\documentclass{article}
\usepackage[utf8]{inputenc}
\usepackage{cite}
\usepackage{amsmath,amssymb,amsfonts}
\usepackage{algorithm}
\usepackage{algorithmic}
\usepackage{graphicx}
\usepackage[caption=false,font=normalsize,labelfont=sf,textfont=sf]{subfig}
\usepackage{textcomp}
\usepackage{comment}
\usepackage{float}
\usepackage{geometry}
\allowdisplaybreaks[4]
\newtheorem{assumption}{\textit{\textbf{Assumption}}}
\newtheorem{theorem}{\textit{\textbf{Theorem}}}
\newtheorem{proposition}{\textit{\textbf{Proposition}}}
\newtheorem{lemma}{\textit{\textbf{Lemma}}}

\newtheorem{remark}{\textit{\textbf{Remark}}}
\newtheorem{proof}{Proof}

\begin{document}
	\title{Decentralized Stochastic Proximal Gradient Descent with Variance Reduction over Time-varying Networks}
	\author{Xuanjie Li, Yuedong Xu
	\thanks{Xuanjie Li and Yuedong Xu are with the School of Information Science and Technology, Fudan University, Shanghai 200237, China (e-mail: \{lixuanjie20, ydxu\}@fudan.edu.cn). }
	, Jessie Hui Wang
	\thanks{J. H. Wang is with the Institute for Network Sciences and Cyberspace, Tsinghua University, Beijing 100084, China, and also with the Beijing National Research Center for Information Science and Technology, Beijing 100084,
		China (e-mail: jessiewang@tsinghua.edu.cn).}
	, Xin Wang
	\thanks{Xin Wang is with the Key Laboratory of EMW Information
		(MoE), Department of Communication Science and Engineering, Fudan
		University, Shanghai 200433, China (e-mail: xwang11@fudan.edu.cn).}
	and John C.S. Lui
	\thanks{John C. S. Lui is with the Department of Computer Science and Engineering, The Chinese University of Hong 	Kong, Hong Kong (e-mail: cslui@cse.cuhk.edu.hk).}	
}
	
	\maketitle

	\begin{abstract}
		
		In decentralized learning, a network of nodes cooperate to minimize an overall objective function that is usually the finite-sum of their local objectives, and incorporates a non-smooth regularization term for the better generalization ability. Decentralized stochastic proximal gradient (DSPG) method is commonly used to train this type of learning models, while the convergence rate is retarded by the variance of stochastic gradients. In this paper, we propose a novel algorithm, namely DPSVRG, to accelerate the decentralized training by leveraging the variance reduction technique. The basic idea is to introduce an estimator in each node, which tracks the local full gradient periodically, to correct the stochastic gradient at each iteration. By transforming our decentralized algorithm into a centralized inexact proximal gradient algorithm with variance reduction, and controlling the bounds of error sequences, we prove that DPSVRG converges at the rate of $O(1/T)$ for general convex objectives plus a non-smooth term with $T$ as the number of iterations, while DSPG converges at the rate $O(\frac{1}{\sqrt{T}})$. Our experiments on different applications, network topologies and learning models demonstrate that DPSVRG converges much faster than DSPG, and the loss function of DPSVRG decreases smoothly along with the training epochs.
		
	\end{abstract}

	\section{Introduction}
	\label{sec:introduction}
	Decentralized algorithms to solve finite-sum minimization problems are crucial to train machine learning models where data samples are distributed across a \emph{network} of nodes. Each of them communicates with its one-hop neighbors, instead of sending information uniformly to a centralized server. The driving forces
	toward decentralized machine learning are twofold. One is the ever-increasing privacy concern \cite{gaia, communication, privacy, privacy1}, especially when the personal data is collected by smartphones, cameras and wearable devices. In fear of privacy leakage, a node is inclined to keeping and computing the raw data locally, and communicating only with other trustworthy nodes. The other is owing to the expensive communication. Traditional distributed training with a centralized parameter server requires that all nodes push gradients and pull parameters so that the ingress and egress links of the server can easily throttle the traffic. By removing this server and balancing the communications in a network, one can reduce the wall-clock time of model training several folds \cite{outperform}. 
	
	Consensus-based gradient descent methods \cite{DGD, network, consensus, ADMM} are widely used in decentralized learning problems because of their efficacy and simplicity. Each node computes the local gradient to update its parameter with the local data, exchanges this parameter with its neighbors, and then calculates the weighted average as the start  of the next training round. To avoid overfitting the data, $\ell_1$ or trace norms are introduced as an additional penalty to the original loss function, yet the new loss function is {\emph{non-smooth} at kinks. Decentralized Proximal Gradient (DPG) \cite{first-order-optimization} leverages a proximal operator to cope with the non-differentiable part. Its stochastic version, Decentralized Stochastic Proximal Gradient (DSPG) \cite{DSPG}, reduces the computation complexity per-iteration by using the stochastic gradient other than the full gradient. 
		
		Stochastic gradient methods tend to have the potentially large variance and their performance relies on the tuning of a decaying learning rate sequence. Several variance reduction algorithms have been proposed to solve this problem in the past decade, e.g. SAGA \cite{SAGA}, SVRG \cite{SVRG}, SCSG \cite{SCSG} and SARAH \cite{SARAH}. 
		Variance reduction is deemed as a breakthrough in the first-order optimization that accelerates Stochastic Gradient Descent (SGD) to the linear convergence rate under smooth and strongly convex conditions. The variance reduction technique is more imperative in decentralized machine learning \cite{heu-vr} for two reasons. Firstly, the local averaging is inefficient to mitigate the noise of the models in decentralized topologies. In light of the limited number of local models available at a node, the averaged model has a much larger variance compared to the centralized setting. Secondly, the intermittent network connectivity in real-world brings the \emph{temporal} variance of stochastic gradients.  
		Hence, the underlying network topology is time-varying that introduces further randomness in the estimation of local gradients temporally. Owing to the above considerations, variance reduction is generally used in decentralized learning, including DSA \cite{DSA}, Network-SVRG/SARAH \cite{network-svrg} and GT-SVRG \cite{gt-svrg}.
		
		Our work also considers \emph{time-varying networks}, where the link connectivity is changing over time. Time-varying networks are ubiquitous in the daily life \cite{time-vary}. For instance, a transmission pair is interrupted if one of the nodes moves out of the mutual transmission range in mobile networks. Besides, only the set of non-interfering wireless links can communicate simultaneously and different sets of links have to be activated in a time-division mode. Time-varying graphs not only allow the dynamic topology, but also are not guaranteed to be connected all the time. The representative works on convex problems include DIGing\cite{diging}, PANDA\cite{panda} and the time-varying $\mathcal{AB}$/push-pull method\cite{push-pull}.

		In this paper, we propose the Decentralized Proximal Stochastic Variance Reduced Gradient descent (DPSVRG) algorithm to address a fundamental problem, i.e., \emph{can the stochastic decentralized learning over temporal changing networks with a general convex but non-smooth loss function be \textbf{faster} and more \textbf{robust}.} 
		DPSVRG introduces a new variable for each node to help upper bound the local gradient variance and updates this variable periodically. During the training, this bound will decrease continually and the local gradient variance will vanish as well. In addition, DPSVRG leverages multi-consensus to further speed up the convergence rate that is applicable to both static and time-varying networks. Up to now, the only decentralized stochastic proximal algorithm \cite{pmgt-vr} considers static graphs and strongly convex global objectives instead.

		We rigorously prove that the convergence rate of DPSVRG is $O(\frac{1}{T})$ toward general convex objective functions in contrast to $O(\frac{1}{\sqrt{T}})$ for DSPG without variance reduction, where $T$ refers to the number of iterations.  
		We reformulate DPSVRG as a centralized inexact proximal gradient algorithm with variance reduction, namely Inexact Prox-SVRG, so that both of them have the same convergence rate. In this process, we delicately transform the parameter difference caused by decentralized environment into two errors, \emph{gradient error} and \emph{proximal error}, in Inexact Prox-SVRG. Through controlling these two errors under some mild conditions, we first prove the $O(\frac{1}{T})$ convergence rate for Inexact Prox-SVRG with a general convex objective function plus a non-smooth term. We subsequently prove that DPSVRG satisfies these conditions as well and achieves the same convergence rate. The algorithm in \cite{stochastic-proximal} is similar to our Inexact Prox-SVRG, but it only considers the proximal errors. The gradient error analysis pertinent to the decentralized learning is of equal significance that demands a different algorithm and proof. DPSVRG and the associated Inexact Prox-SVRG are based on general convex functions, while most of SVRG related studies \cite{vr-reduced, optimal, prox-svrg} target at strongly convex ones; in other word, our assumption is weaker and DPSVRG certainly has a broader applicability.


		DPSVRG is implemented and evaluated on a testbed with eight nodes. Taking DSPG as our baseline, we conduct comprehensive experiments on four datasets with the logistic regression plus an $\ell_1$ regularizer. Experimental results show that DPSVRG has a much faster convergence rate and achieves higher accuracy within the same training epochs. We separately evaluate our variance reduction and multi-consensus techniques to show their impact on the convergence properties. DPSVRG also outperforms DSPG with different regularization coefficients and network topologies, thereby exhibiting the robustness of DPSVRG.
		
		The remainder of this paper is structured as follows. Section \ref{sec:model} presents the mathematical model. We propose a noval DPSVRG algorithm in Section \ref{sec:algorithm}. Section \ref{sec:convergence analysis} proves the sublinear convergence of the proposed algorithm and Section \ref{sec:experiments} provides experimental results to validate its efficacy. Section \ref{sec:conclusion} concludes this work.

		\section{Mathematical Model}
		\label{sec:model}
		In this section, we present the network model and basic assumptions on decentralized machine learning with a class of non-smooth objective functions. 
		
		\subsection{Network Model}
		\label{subsec:network model}
		
		\textbf{Time-varying graph.}  
		Consider an undirected time-varying graph sequence $\mathcal{G}^t= \{\mathcal{V}, \mathcal{E}^t\}$, where $\mathcal{V}$ is the set of nodes with $|\mathcal{V}| = m$ and $\mathcal{E}^t$ is the set of undirected edges at slot $t$, which is changeable over time. The pair-wise nodes connected by an edge are able to communicate with each other. We denote $\mathcal{N}_i^t$ as the set of neighbors of node $i$ at time $t$ including node $i$ itself, i.e. $\mathcal{N}_i^t = \{j \ | \ (i,j) \in \mathcal{E}^t\} \cup \{i\}$. The connectivity of graph $\mathcal{G}^t$ plays a key role in the decentralized learning, and here we first state a couple of well accepted assumptions. 
		
		
		\begin{assumption} 
			\label{assum:connectivity}
			(Graph Connectivity) Graph sequence $\mathcal{G}^t$ is said to be $b$-connected when the node set $\mathcal{V}$ with the union of consecutive edge sets  
			$\bigcup_{t=jb}^{(j+1)b-1} \mathcal{E}^{t}$ is connected, i.e., each node can reach other nodes at every round $j = 0, 1, 2, \cdots$. 
			\label{ass:connectivity}
		\end{assumption}
		
		Assumption \ref{ass:connectivity} guarantees that the information of one node can reach any other node in finite time, which is weaker compared to the assumption of persistent connectivity. We define a non-negative time-varying matrix $\mathbf{W}^t \in \mathbb{R}^{n \times n}$, where $w_{ij}^t = 0$ implies that node $i$ and node $j$ are not neighbors at time $t$. We make the following assumption on this matrix as below. 
		
		\begin{assumption} 
			\label{assum:doubly}
			(Doubly Stochastic Matrix) Each matrix $\mathbf{W}^t$ is doubly stochastic that satisfies
			\begin{equation*}
				\mathbf{W}^t \cdot \mathbf{1}=\mathbf{1} , \quad (\mathbf{W}^t)^\mathsf{T}\mathbf{1}=\mathbf{1} ,
			\end{equation*}
			with $\mathbf{1} = [1,...,1]^\mathsf{T}$. Any non-zero entry in $\mathbf{W}^t$ is no less than a positive constant $\eta$, i.e. 
			\begin{equation*}
				w_{ij}^t 
				\left\{
				\begin{aligned}
					&\ge \eta , & & \textrm{if} \;\;w_{ij}^t > 0  \\
					& = 0, & & \textrm{otherwise} \\
				\end{aligned}
				\right ..
			\end{equation*}
			\label{ass:doubly}		
		\end{assumption}

		Matrices satisfying Assumption \ref{assum:doubly} are widely utilized in the decentralized learning as \textit{communication} matrices to capture how the parameters from different nodes are aggregated.  Specifically, after receiving the parameters from its neighbors, node $i$ uses the weighted sum, specified by $W$, to update its out-of-date parameter. Therefore, the weight threshold in Assumption \ref{ass:doubly} guarantees that each node will always exert nonnegligible influences on its neighbors. In this way, the node can acquire sufficient information from its neighbors and minimize the global objective cooperatively. Given a sequence of doubly stochastic matrices spanning from $W^{l}$ to $W^{g}$ with $l \le g$, we denote $\Phi(l,g)$ as
		\begin{equation*}
			\Phi(l,g) = W^{g} W^{g-1}\cdots W^{l}.
		\end{equation*}
		Let $\phi_{ij}^{(l,g)}$ be the entry in the $i^{th}$ row and the $j^{th}$ column of $\Phi(l,g)$. The asymptotic property of $\Phi(l,g)$ implies that all the elements $\phi_{ij}^{(l,g)}$ will be distributed in the vicinity of $\frac{1}{m}$ as $t \rightarrow \infty$. Formally, we introduce this known result as below.
		\begin{lemma}
			\label{lemma:matrix converge}\cite{DGD}
			Given a $b$-connected graph sequence $(\mathcal{V}, \mathcal{E}^t) $ with $|\mathcal{V}| = m$ and a sequence of doubly stochastic matrices  $\{\mathbf{A}^t\}$, if each matrix $\mathbf{A}^t$ holds the property $|\sigma_2(\mathbf{A}^t)| < 1$, where $\sigma_2(\mathbf{A}^t)$ is the second largest eigenvalue of $\mathbf{A}^t$, then the following inequality holds
			\begin{equation}
				\label{eq:consensus}
				\left| \phi_{ij}^{(l,g)} - \frac{1}{m} \right| \le \Gamma \gamma^{g-l},
			\end{equation}
			for all $g,l$ with $g \ge l$. Here $\Gamma = 2 (1+\eta^{-b_0})$, $\gamma = 1-\eta^{b_0}$ and $b_0 = (m-1)b$.
		\end{lemma}
		The above lemma states that all elements of $\Phi(l,g)$ will converge to $\frac{1}{m}$ as $(g-l) \rightarrow \infty$. With this feature, the matrix $\Phi$ is utilized as the \textit{aggregated communication} matrix in our subsequent multi-consensus setting in which the nodes perform multiple communications with changeable matrices $\mathbf{W}^t$  in one training iteration.

		\subsection{Decentralized Optimization Model}
		
		Our machine learning model possesses two prominent features that constitutes the key novelty of our work. 
		
		\textbf{(1) Handling non-smoothness.} Regularization is a \emph{de facto} standard to resolve the overfitting problem in many regression and classification problems. The basic idea of regularization is to introduce an additional penalty term in the loss function in which 
		the $\ell_1$ norm or the trace norm of weights are commonly used. Since these norms are non-differentiable at kink(s), the proximal gradient descent method is usually adopted to train these learning models.
		
		\textbf{(2) Decentralization.} Due to the privacy concern, each node is usually reluctant to sharing its raw data to any other node. An alternative approach is to perform the training locally and distributedly while sharing the model parameters. 
		When the communication between a subset of nodes is not allowed, either because of restricted network  connections or social distrust, distributed learning degenerates to decentralized learning on a graph that no node plays the role of a parameter server. 
		
		In our decentralized learning, all $m$ nodes in $\mathcal{G}$ are jointly solving a global optimization problem. Each node $i$ possesses a local dataset  $\mathcal{D}_i$ with $n_i$ data records and a copy $\mathbf{x}_i$ of the variable $\mathbf{x} \in \mathbb{R}^d$. The local dataset and parameter construct a local loss function $F_i$ and all the nodes cooperatively minimize the average of all the loss functions,
		\begin{equation*}
			\label{P1}
			\begin{aligned}
				&\min_{\mathbf{x}\in\mathbb{R}^d} F(\mathbf{x}) = \frac{1}{m}\sum_{i=1}^m F_i(\mathbf{x}),
			\end{aligned}
			\tag{P1}
		\end{equation*}
		where $F_i(\mathbf{x})$ is a composite function, i.e.,
		\begin{equation*}
			F_i(\mathbf{x}) = \frac{1}{n_i}\sum_{j=1}^{n_i}f(\mathbf{x}; \zeta_i^j) + h(\mathbf{x}),
		\end{equation*}
		and $\zeta_i^j$ denotes the data sampled from the local training dataset $\mathcal{D}_i$. Without loss of generality, $f(\cdot)$ refers to the loss of the training objective, and $h(\cdot)$ symbolizes any regularizer that is non-smooth (e.g. $\ell_1$ norm).

		To facilitate qualitative studies, we make the following assumptions. Unless otherwise specified, we use $\Vert \cdot \Vert$ and $\langle \cdot,\cdot \rangle$ to denote the $\ell_2$ norm of vectors and the inner product of two vectors respectively throughout this paper.
		
		\begin{assumption}
			\label{ass:f-property}
			The function $f(\mathbf{x})$ is convex and differentiable, with an $L$-Lipschitz continuous gradient on the open set $\mathbb{R}^d$, i.e., for all $\mathbf{x}, \mathbf{y} \in \mathbb{R}^d$, 
			\begin{equation}
				\label{eq:f convex}
				f(\mathbf{x};\cdot) \ge f(\mathbf{x};\cdot) + \langle \nabla f(\mathbf{x};\cdot), \mathbf{y}-\mathbf{x}\rangle,
			\end{equation}
			and for all $\mathbf{x}, \mathbf{y} \in \mathbb{R}^d$,
			\begin{equation}
				\label{eq:smooth}
				\Vert \nabla f(\mathbf{y};\zeta_i^l)-\nabla f(\mathbf{x};\zeta_i^l) \Vert \le L \Vert \mathbf{y}-\mathbf{x} \Vert ,
			\end{equation}
			where $\nabla$ is the derivative operator.
		\end{assumption}
		
		\begin{assumption}
			\label{ass:h-property}
			The function $h$ is lower semi-continuous and convex, and its effective domain ${\rm dom} \ h = \{ \mathbf{x}\in\mathbb{R}^d|h(\mathbf{x})<+\infty\}$ is closed. Similarly, for $\forall \mathbf{x} \in {\rm dom} \ h, \ \forall y \in \mathbb{R}^d$, we have
			\begin{equation}
				h(\mathbf{y}) \ge h(\mathbf{x}) + \langle \mathbf{\xi}, \mathbf{y}-\mathbf{x} \rangle, \quad \forall \mathbf{\xi} \in \partial h(\mathbf{x}),
			\end{equation}
			where $\partial$ denotes the subgradient operator.
		\end{assumption}
		
		For the proper convex function $h(\cdot)$, the closedness and the lower semi-continuity have been proved to be equivalent \cite{optimization-models}.  With Assumption \ref{ass:h-property}, $h(\cdot)$ is a closed function within its effective domain and we will show that this facilitates the proximal operator to tackle the non-smooth problem.

		\textbf{Example.} To better interpret the above assumptions, let us present a concrete example.
		Consider the decentralized least-squares method plus an $\ell_1$ norm regularizer in regression or classification problems. The loss function takes the form $f(\mathbf{w})=\sum_{i=1}^N \Vert \mathbf{a}_i^\mathsf{T} \mathbf{w} - \mathbf{b}_i \Vert_2^2$ on a dataset consisting of $N$ points $(\mathbf{a}_i, \mathbf{b}_i)$. 
		The regularization term is $h(\mathbf{w})=\Vert\mathbf{w}\Vert_1 = \sum_j|\mathbf{w}_j|$. It is easy to check that $f$ is convex and $L$-smooth with $L = 2 \max_{i\in\{1,...,N\}} {\Vert \mathbf{a}_i \mathbf{a}_i^\mathsf{T}}\Vert$, and $h$ is lower semi-continuous and convex.

		\section{DPSVRG Algorithm}
		\label{sec:algorithm}
		In this section, we propose a novel DSPG algorithm with variance reduction in decentralized networks. The procedures of updating and transmitting parameters are elaborated. 
		
		\subsection{Variance Reduction for Decentralized SGD} 
		
		In the most decentralized algorithms, each node makes use of local data to update the local parameter to approximate the local minimum. Considering the high computation cost of local full gradient computing, stochastic gradient descent is more practical to calculate the approximation of the actual local gradient from a randomly selected subset of the data. We begin with a smooth objective function and then extend it to a non-smooth function shortly after.

		At each iteration $k$, node $i$ first samples a batch of data $\mathcal{B}_i \subset \mathcal{D}_i$. The average gradient $\mathbf{v}_i^{(k)}$ of the objective function $f(\mathbf{x}_i^{(k-1)};\cdot)$ 
		with regard to data batch $\mathcal{B}_i$ is taken as the estimate of the actual local full gradient 
		\begin{equation}
			\label{eq:local grad}
			\mathbf{v}_i^{(k)} = \frac{1}{|\mathcal{B}_i|}\sum_{\zeta_i^j \in\mathcal{B}_i} \nabla f(\mathbf{x}_i^{(k-1)};\zeta_i^j).
		\end{equation}
		The standard update formula in the elementary SGD is presented as
		\begin{equation}
			\label{eq:local update}
			{\mathbf{x}_i^{(k)} = \mathbf{x}_i^{(k-1)} -\alpha \mathbf{v}_i^{(k)}},
		\end{equation}
		where $\alpha$ denotes the step size. For ease of notation, we rewrite 
		$\frac{1}{|\mathcal{B}_i|}\sum_{\zeta_i^j \in \mathcal{B}_i} \nabla f(\cdot;\zeta_i^j)$ as $\nabla f_i^{\mathcal{B}_i}(\cdot)$ and the local gradient on the dataset $\mathcal{D}_i$ is conveniently expressed as $\nabla f_i(\cdot)$. It is worthy to note that $\mathbf{v}_i^{(k)}$ is an unbiased estimator of $\nabla f_i(\mathbf{x}_i^{(k-1)})$, i.e., $\mathbb{E}_{\mathcal{B}_i}[\mathbf{v}_i^{(k)}] = \nabla f_i(\mathbf{x}_i^{(k-1)})$.

		\textbf{Root cause of variance.} Although $\mathbf{v}_i^{(k)}$ is an unbiased estimator of the local gradient,  it reflects an inaccurate direction deviating from the local full gradient at some iterations, thus increasing the variance and may impact the convergence rate. This motivates us to study how to accelerate SGD by mitigating the variance of $\mathbf{v}_i^{(k)}$. 
		
		\textbf{Suppressing Variance.}
		We describe the basic idea of stochastic variance reduction gradient (SVRG) approach and generalize it to a decentralized network. In SVRG, the model update is conducted in two loops. The outer loop captures the full local gradient on all the local data, while the inner loop computes the gradient with a random batch and utilizes the full gradient to correct the batch gradient for the purpose of suppressing its variance. 
		
		\textbf{Why does the correction work?} 
		The original idea of variance reduction stems from basic statistics. Consider three random variables $X$, $Y$ and $Z$ satisfying $Z = (X-Y) + \mathbb{E}Y$. Obviously $Z$ is an unbiased estimator of $\mathbb{E}X$, and the variance ${\rm Var}(Z)$ is smaller than ${\rm Var}(X)$ as long as $X$ and $Y$ exhibit a strong positive correlation. 	
		In SVRG, $X$ represents the local gradient sampled on a batch of data, i.e. $\nabla f_i^{\mathcal{B}_i}(\mathbf{x}_i^{(k-1,s)})$, with $k$ and $s$ indicating the inner and outer steps respectively. Given the parameter $\tilde{\mathbf{x}}_i^{s-1}$ at the $s$-th outer round, the local full gradient is $\nabla f_i(\tilde{\mathbf{x}}_i^{s-1})$ and the stochastic gradient on the  batch $\mathcal{B}_i$ is $\nabla f_i^{\mathcal{B}_i}(\tilde{\mathbf{x}}_i^{s-1})$. Here, $Y$ is equivalent to $\nabla f_i^{\mathcal{B}_i}(\tilde{\mathbf{x}}_i^{s-1})$, and $\mathbb{E}Y$ is $\nabla f_i(\tilde{\mathbf{x}}_i^{s-1})$. Although $X$ and $Y$ are different random variables, they are sampled on the same batch of data and the differences between $\mathbf{x}_i^{(k-1,s)}$ and $\tilde{\mathbf{x}}_i^{s-1}$ are relatively small so that they are highly correlated. When introducing the bias $(Y-\mathbb{E}Y)$ in the inner-loop local update, the variance of local gradient estimator is abated. To summarize, the update steps are given below
		\begin{equation*}
			\label{eq:vr grad}
			\begin{aligned}
				&\quad \mathbf{v}_i^{(k,s)} = \nabla f_i^{\mathcal{B}_i}(\mathbf{x}_i^{(k-1,s)}) - (\nabla f_i^{\mathcal{B}_i}(\tilde{\mathbf{x}}_i^{s-1}) - \nabla f_i(\tilde{\mathbf{x}}_i^{s-1})) \\
				&\quad \mathbf{x}_i^{(k,s)} \leftarrow \mathbf{x}_i^{(k-1,s)} -\alpha \mathbf{v}_i^{(k,s)}
			\end{aligned}
			\tag{Gradient Step}.
		\end{equation*}
		
		\textbf{Properties of SVRG.} 
		Due to this variance-reduction mechanism, the convergence rate of DPSVRG is independent of the variance of local stochastic gradients. The improvement of convergence rate will be more evident when the sampled batch $\mathcal{B}_i$ is relatively small, which leads to large gradient variance.

		\subsection{Decentralized Model Fusion}
		\label{subsec:model fusion}
		
		In our decentralized network, each node $i \in \mathcal{V}$ maintains its local copy $\mathbf{x}_i$ and executes gradient descent step as mentioned before. Owing to the data disparity of distributed nodes, parameters $x_i$ descend towards diverse directions and reach different values. Our goal is to achieve a \emph{consensus} of all the local parameters to attain a global optimum $\mathbf{x}^*$. To meet this objective, each node communicates with its direct neighbors to exchange their updated parameters and takes the weighted average over its neighborhood prescribed by a time-varying matrix. This process is known as \emph{gossip} \cite{gossip}. Given a doubly stochastic matrix $\mathbf{W}^{(k,s)}$ at iteration $(k,s)$, a single step of gossip can be stated as
		\begin{equation}
			\label{eq:gossip}
			\mathbf{x}_i^{(k,s)} \leftarrow \sum_{j \in \mathcal{N}_i} \mathbf{W}_{ij}^{(k,s)} \mathbf{x}_j^{(k,s)}, \quad \forall i \in \mathcal{V}.
		\end{equation}
		
		\textbf{Multi-consensus.}
		When the communication matrix $\mathbf{W}$ is fixed over time, the convergence rate of decentralized learning is largely dependent on the spectral gap of $W$ and is higher with more communication steps within an iteration round \cite{outperform, multiconsensus-stochastic, multiconsensus-improved}. This insight motivates us to utilize multi-consensus, i.e., nodes performs the gossip step for $k$ times consecutively, where $k$ denotes the inner step. Therefore, the consensus step at iteration $(k,s)$ can be presented as follows,
		\begin{equation*}
			\mathbf{x}_i^{(k,s)} = \sum_{j=1}^m \phi_{ij}^{(k,s)} \mathbf{x}_j^{(k,s)}, \tag{Consensus Step}
		\end{equation*}
		where $\phi_{ij}^{(k,s)}$ denotes the $i$-th row and the $j$-th column element of a combination of $k$ different communication matrices, i.e., $\phi_{ij}^{(k,s)} = [\mathbf{W}^{(k,s),1} \mathbf{W}^{(k,s),2} ... \mathbf{W}^{(k,s),k}]_{i,j}$. When $k$ is large, Lemma \ref{lemma:matrix converge} indicates that each element in matrix $\Phi(k,s)$ approaches $\frac{1}{m}$, which implies that the \emph{consensus step} impels all the local parameters to the average parameter. Therefore, our algorithm is getting closer to the centralized counterpart as iteration proceeds, which retains higher convergence rate compared with decentralized ones. Furthermore, we prove a linear convergence rate in regard to $s$ for our algorithm, which can compensate for the high communication cost introduced by multi-consensus. Indeed, we show that DPSVRG still converges faster than the baseline algorithm with the same communication rounds.

		\subsection{Proximal Operator}
		\label{sebsec:proximal operator}

		\textbf{Proximal Operator.}
		Up to now, we only focus on the differentiable part $f$ in (\ref{P1}). A special operation, namely \emph{proximal operator}, is needed to dispose the non-smooth function $h$. With the properties in \textit{Assumption \ref{ass:h-property}}, for a point $\mathbf{z} \in \mathbb{R}^d$ and a scalar $\alpha$, the proximal operator is defined as,
		\begin{equation}
			\label{eq:prox}
			{\rm prox}_h^{\alpha} \{\mathbf{z}\} = \arg\min_{\mathbf{y}\in {\rm dom} \ h} \left\{ \frac{1}{2\alpha}\Vert \mathbf{y}-\mathbf{z} \Vert^2 + h(\mathbf{y}) \right\}.
		\end{equation}
		As an essential part of Proximal Gradient Descent, the proximal operator is applied after the primitive gradient step $\mathbf{z} = \mathbf{x} - \alpha \nabla f(\mathbf{x})$ to map the original point $\mathbf{x}$ into a specific region. For example, by setting $h(\mathbf{y}) = \Vert \mathbf{y} \Vert_1$, Equation \eqref{eq:prox} yields the vector which is close to $\mathbf{z}$ and has small absolute values.
		
		\textbf{Proximal Mapping Step.} 
		Motivated by Proximal Gradient Method, we apply the proximal mapping after the consensus step to deal with the non-differentiable part $h$. Node $i$ substitutes $\mathbf{z}$ with the outcome of the consensus step $\mathbf{x}_i^{(k,s)}$, and executes the following update procedure
		\begin{equation*}
			\mathbf{x}_i^{(k,s)} \leftarrow {\rm prox}_h^{\alpha}\{\mathbf{x}_i^{(k,s)}\}. \tag{Proximal Mapping Step}
		\end{equation*}
		
		The following lemmas guarantee the availability of the proximal operator and present several useful properties.
		\begin{lemma} \cite{first-order-optimization}
			(First Prox Theorem) Let $h$ be a proper closed and convex function, then ${\rm prox}_h$ is a singleton for any $\mathbf{x} \in {\rm dom} \ h$.
		\end{lemma}
		\begin{lemma}
			\label{lemma:prox-property}
			(Second Prox Theorem) Let $h$ be a proper closed and convex function, then for any $\mathbf{y},\mathbf{z} \in {\rm dom} \ h$, the following statements are equivalent. \\
			(1) $\mathbf{y} = {\rm prox}_h^{\alpha}\{\mathbf{z}\}$ \\
			(2) $\frac{1}{\alpha}(\mathbf{z} - \mathbf{y}) \in \partial h(\mathbf{y})$ \\
			(3) $ \frac{1}{\alpha} \langle \mathbf{z} - \mathbf{y}, \mathbf{x} - \mathbf{y} \rangle \le h(\mathbf{x}) - h(\mathbf{y}), \ \forall \mathbf{x}\in {\rm dom} \ h$.
		\end{lemma}
		\begin{lemma}
			\label{lemma:nonexpansive}
			Let $h$ be a proper closed and convex function. For any two points $\mathbf{z}_1, \mathbf{z}_2 \in {\rm dom} \ h$, we have
			\begin{equation*}
				\Vert {\rm prox}_{h}^\alpha \{\mathbf{z}_1\} - {\rm prox}_{h}^\alpha \{\mathbf{z}_2\} \Vert \le \Vert \mathbf{z}_1-\mathbf{z}_2 \Vert.
			\end{equation*}
		\end{lemma}
		
		The closedness of ${\rm dom} h$ guarantees the availability of proximal operator solution and the First Prox Theorem ensures the uniqueness of this solution. The second Prox Theorem exhibits some important properties of $h$ which will be utilized in our convergence analysis.
		
		\textbf{Practicability of Proximal Operator.}
		Next, we check whether the closed-form expressions of the proximal operators can be obtained and applied in general machine learning tasks. In fact, the closed-form expressions of proximal operators are available in many cases of interest. For example, the closed-form expression of proximal operator of $\ell_1$ norm is given by
		\begin{equation*}
			\left[{\rm prox}_{\lambda \Vert \cdot \Vert_1}^\alpha \{\mathbf{y}\}\right]_i = 
			\left\{
			\begin{aligned}
				&\mathbf{y}_i - \alpha \lambda , & &\mathbf{y}_i > \alpha \lambda \\
				&0, & &|\mathbf{y}_i| \le \alpha \lambda \\
				&\mathbf{y}_i + \alpha \lambda , & &\mathbf{y}_i < -\alpha \lambda \\
			\end{aligned}
			\right. ,
		\end{equation*}
		where $[\cdot]_i$ denotes the $i$-th element of the vector. Therefore, our proposed algorithm can be generalized to a wide range of machine learning scenarios.

		\subsection{Algorithm Design}
		\label{subsec:algorithm design}
		
		Bringing \emph{gradient step}, \emph{consensus step} and \emph{proximal mapping step} together, we present the sketch of our DPSVRG algorithm in Algorithm \ref{alg:prox-svrg}. For ease of explanation, we enumerate the extended operations beyond the classical SVRG algorithm \cite{SVRG}:
		
		\begin{itemize}
			\item We adopt $\mathbf{q}_i^{(k,s)}$ and $\hat{\mathbf{q}_i}^{(k,s)}$ to denote the intermediate result of \emph{gradient step} and \emph{consensus step} respectively.
			
			\item Each node, indexed by $i$, only samples a single data record instead of a batch in every training round. The notation $\nabla f_i^{l_i}(\mathbf{x}_i)$ stands for the gradient of function $f$ at $\mathbf{x}_i$ with data $l_i$.
			
			\item The number of inner steps $K_s$ increases exponentially with the base $\beta$ and the exponent $s$.
			
			\item Except for the initial state, $\tilde{\mathbf{x}}_i^s$ always differs from $\mathbf{x}_i^{(0,s+1)}$.
		\end{itemize}

		\begin{algorithm}[htbp]
			\caption{DPSVRG \textit {at node $i$}}
			\label{alg:prox-svrg}
			\begin{algorithmic}[1]
				\STATE Initiate $\mathbf{x}^{init}$
				\STATE $\tilde{\mathbf{x}}_i^0 = \mathbf{x}^{init}$, $\mathbf{x}_i^{(0,1)} = \mathbf{x}^{init}$
				\FOR{$s = 1, 2, ..., S$}
				\STATE $K_s = \lceil\beta^sn_0\rceil$
				\STATE compute $\nabla f_i(\tilde{\mathbf{x}}_i^{s-1})$
				\FOR{$k=1, 2, ..., K_s$}
				\STATE randomly pick $l_i \in \{1, 2, ..., n_i\}$
				\STATE $\mathbf{v}_i^{(k,s)} = \nabla f_i^{l_i}(\mathbf{x}_i^{(k-1, s)}) - \nabla f_i^{l_i}(\tilde{\mathbf{x}}_i^{s-1}) + \nabla f_i(\tilde{\mathbf{x}}_i^{s-1})$
				\STATE $\mathbf{q}_i^{(k,s)} = \mathbf{x}_i^{(k-1,s)}-\alpha \mathbf{v}_i^{(k,s)}$
				\STATE $\hat{\mathbf{q}_i}^{(k,s)} = \sum_{j=1}^{m}\phi_{i,j}^{(k,s)}\mathbf{q}_j^{(k,s)}$
				\STATE $\mathbf{x}_i^{(k,s)} = {\rm prox}_h^{\alpha}\{\hat{\mathbf{q}_i}^{(k,s)}\}$
				\ENDFOR
				\STATE $\tilde{\mathbf{x}}_i^s = \frac{1}{K_s}\sum_{k=1}^{K_s}\mathbf{x}_i^{(k,s)}$
				\STATE $\mathbf{x}_i^{0,s+1} = \mathbf{x}_i^{K_s,s}$
				\ENDFOR
			\end{algorithmic}
		\end{algorithm}
		
		Qualitatively analyzing the convergence of DPSVRG is very challenging. To address this challenge, we propose to transform it into a \emph{centralized inexact proximal gradient method with variance reduction}. 
		
		
		\textbf{Inexact Proximal Gradient} 
		introduces an error term into proximal operator and attains an inexact version of proximal operator defined below. For a function $h$ satisfying \textit{Assumption \ref{ass:h-property}}, a point $\mathbf{z} \in \mathbb{R}^d$ and a scalar $\alpha$, we claim $\mathbf{x} = {\rm prox}_{h,\varepsilon}^\alpha \{\mathbf{z}\}$, if there exists a scalar $\varepsilon$ such that
		\begin{equation}
			\label{eq:inexact prox}
			\frac{1}{2\alpha} \Vert \mathbf{x}-\mathbf{z} \Vert^2 + h(\mathbf{x}) \le \min_{\mathbf{y}\in {\rm dom} \ h} \left\{ \frac{1}{2\alpha}\Vert \mathbf{y}-\mathbf{z} \Vert^2 + h(\mathbf{y}) \right\} + \varepsilon.
		\end{equation}
		The error term $\varepsilon$ indicates the difference of function values between the inexact proximal operation solution and the exact solution. 
		
		\textbf{Algorithm Reformulation.}
		We aim to construct a centralized algorithm, assuming that there is a \emph{virtual} central node which stores all the training data $\mathcal{D} = \{\mathcal{D}_i\}, i\in [1,m]$ and takes the average $\bar{\mathbf{x}} = \frac{1}{m} \sum_{i=1}^m \mathbf{x}_i$ as its parameter. This virtual node \emph{tracks} the average of parameters of distributed nodes in Algorithm \ref{alg:prox-svrg} at each training round by executing the centralized Inexact Prox-SVRG algorithm in Algorithm \ref{alg:inexact prox-svrg}. Accordingly, this virtual node solves \eqref{P1} with a different form
		\begin{equation*}
			\label{P2}
			\begin{aligned}
				&\min_{\mathbf{x}\in\mathbb{R}^d} F(\mathbf{x}) = \underbrace{\frac{1}{n}\sum_{j=1}^{n}f(\mathbf{x}; \zeta^j)}_{f(\mathbf{x})} + h(\mathbf{x})
			\end{aligned},
			\tag{P2}
		\end{equation*}
		where $\zeta^j \in \mathcal{D}$. In Algorithm \ref{alg:inexact prox-svrg}, notation $l_{in}$ represents the union set of sampled data $l_i$ in Algorithm \ref{alg:prox-svrg}, i.e., $l_{in} = \{l_1,...,l_m\}$, and $f^{l_{in}}(x) = \frac{1}{m} \sum_{j\in l_{in}} f^j(\mathbf{x}) = \frac{1}{m} \sum_{j\in l_{in}} f(\mathbf{x};\zeta^j)$. Therefore, we can set up a relationship of function values between the decentralized and centralized algorithms, i.e., $f^{l_{in}}(\mathbf{x}) = \frac{1}{m} \sum_{i=1}^{m} f_i^{l_i}(\mathbf{x})$. The terms $\mathbf{e}^{(k,s)}$ and $\varepsilon^{(k,s)}$ denote the errors of gradient computation and proximal operator respectively. The variable $\tilde{x}$ is similar to $\tilde{x}_i$ in Algorithm \ref{alg:prox-svrg}, which is introduced to reduce the gradient variance.
		
		Theorem \ref{thm:transform} justifies this transform procedure and guarantees their equivalence on the optimal value. The detailed proof of Theorem \ref{thm:transform} is shown in the supplementary.
		
		\begin{theorem}
			\label{thm:transform}
			Algorithm \ref{alg:prox-svrg} and Algorithm 
			\ref{alg:inexact prox-svrg} have the same optimal solution if we set (in Algorithm \ref{alg:inexact prox-svrg})
			\begin{subequations}
				\label{eq:e and epsilon}
				\begin{align}
					\label{eq:e}
					\mathbf{e}^{(k,s)}
					&= \frac{1}{m}\sum_{i=1}^m \Big(\big(\nabla f_i^{l_i}(\mathbf{x}_i^{(k-1, s)}) - \nabla f_i^{l_i}(\bar{\mathbf{x}}^{(k-1,s)})\big) + \notag\\
					\big(\nabla f_i^{l_i}&(\tilde{\mathbf{x}}^{s-1}) - \nabla f_i^{l_i}(\tilde{\mathbf{x}}_i^{s-1})\big)
					+ \big(\nabla f_i(\tilde{\mathbf{x}}_i^{s-1}) - \nabla f_i(\tilde{\mathbf{x}}^{s-1})\big)\Big),
				\end{align}
				\vspace{-15pt}
				\begin{equation}
					\label{eq:varepsilon}
					\begin{aligned}
						\varepsilon^{(k,s)} &= \frac{1}{2\alpha}\Vert\bar{\mathbf{x}}^{(k,s)}-\mathbf{y}^{(k,s)}\Vert^2 \\
						&+ \langle \bar{\mathbf{x}}^{(k,s)}-\mathbf{y}^{(k,s)}, \frac{1}{\alpha}(\mathbf{y}^{(k,s)}-\bar{\mathbf{q}}^{(k,s)}) + \mathbf{p} \rangle,
					\end{aligned}
				\end{equation}
			\end{subequations}
			with $\mathbf{p} \in \partial h(\bar{\mathbf{x}}^{(k,s)})$, $\bar{\mathbf{q}}^{(k,s)} = \frac{1}{m} \sum_{i=1}^m \hat{\mathbf{q}}_i^{(k,s)} = \frac{1}{m} \sum_{i=1}^m \mathbf{q}_i^{(k,s)}$ and $\mathbf{y}^{(k,s)} = {\rm prox}_h^{\alpha} \{\bar{\mathbf{q}}^{(k,s)}\}$.
		\end{theorem}

		\begin{algorithm}[t] 
			\caption{Inexact Prox-SVRG}
			\label{alg:inexact prox-svrg}
			\begin{algorithmic}[1]
				\STATE Initiate $\mathbf{x}^{init}$
				\STATE $\tilde{\mathbf{x}}^0 = \mathbf{x}^{init}$, $\mathbf{x}^{(0,1)} = \mathbf{x}^{init}$
				\FOR{$s = 1, 2, 3, ..., S$} 
				\STATE $K_s = \lceil \beta^s n_0 \rceil$
				\FOR{$k=1, 2, ..., K_s$}
				\STATE pick $l_{in} = \{l_1, ..., l_m\}$
				\STATE $\mathbf{v}^{(k, s)} = \nabla f^{l_{in}}(\mathbf{x}^{(k-1, s)}) - \nabla f^{l_{in}}(\tilde{\mathbf{x}}^{s-1}) + \nabla f(\tilde{\mathbf{x}}^{s-1})$
				\STATE $\mathbf{q}^{(k,s)} = \mathbf{x}^{(k-1,s)}-\alpha (\mathbf{v}^{(k,s)}+\mathbf{e}^{(k,s)})$
				\STATE $\mathbf{x}^{(k,s)} = {\rm prox}_{h, \varepsilon^{(k,s)}}^{\alpha}\{\mathbf{q}^{(k,s)}\}$
				\ENDFOR
				\STATE $\tilde{\mathbf{x}}^s = \frac{1}{K_s}\sum_{k=1}^{K_s}\mathbf{x}^{(k,s)}$
				\STATE $\mathbf{x}^{(0,s+1)} = \mathbf{x}^{(K_s,s)}$
				\ENDFOR
			\end{algorithmic}
		\end{algorithm}
		
		The importance of Theorem \ref{thm:transform} is embodied in the transform of the complex and decentralized DPSVRG into the inexact Prox-SVRG algorithm without a network structure. The error sequences $\mathbf{e}^{(k,s)}$ and $\varepsilon^{(k,s)}$ in Algorithm \ref{alg:inexact prox-svrg} imply the gap between decentralized and centralized proximal gradient algorithms. More specifically, both errors are generated from the dissensus of local parameters of distributed nodes as follows.
		
		\begin{itemize}
			\item \emph{Gradient error} ($\mathbf{e^{(k,s)}}$) presents the difference from the average of distributed gradients and the gradient of the average parameter, i.e., $\frac{1}{m}\sum_{i=1}^m \mathbf{v}_i^{(k,s)} - \mathbf{v}^{(k,s)}$. This deviation decreases as local copies $\mathbf{x}_i^{(k,s)}$ come close to each other. Furthermore, when the consensus is achieved, i.e., $\mathbf{x}_i=\mathbf{x}_j$ and $\tilde{\mathbf{x}}_i = \tilde{\mathbf{x}}_j, \forall i \neq j$, $\mathbf{e}^{(k,s)}$ becomes zero and the gradients utilized in two algorithms are the same from a global view even without the adjustment $\mathbf{e}^{(k,s)}$.
			\item \emph{Proximal mapping error} ($\mathbf{\varepsilon^{(k,s)}}$) mainly describes the discrepancy between the average of distributed proximal mapping outcomes $\frac{1}{m} \sum_{i=1}^m {\rm prox}\{\mathbf{q}_i^{(k,s)}\}$ and the proximal mapping mapping of global average parameter ${\rm prox} \{ \frac{1}{m} \sum_{i=1}^m \mathbf{q}_i^{(k,s)}\}$. Again, with local parameters tending towards consensus, $\mathbf{q}_i$ will approach $\bar{\mathbf{q}}$ and $\varepsilon^{(k,s)}$ will converge to zero.
		\end{itemize}
		
		To emphasize, the error terms are crucial to the convergence of two algorithms. Moreover, the variance reduction technique renders greater influence of error terms on the upper bound than that of the classical Inexact Proximal Gradient descent (IPG) algorithm \cite{inexact}, demanding for new techniques to control there errors to guarantee the convergence of two algorithms. In what follows, we will show how the variance reduction influences these two error terms and the convergence of our algorithm. Furthermore, we will probe into whether the error terms can be controlled so that DPSVRG achieves the same order of convergence rate as Inexact Prox-SVRG.

		\section{Convergence Analysis}
		\label{sec:convergence analysis}
		
		\subsection{Convergence of Inexact Prox-SVRG}
		
		Inexact Prox-SVRG assumes the existence of a virtual centralized node storing the entire dataset and performing all the computations to solve \eqref{P2}. For convenience, we consider the situation that a single data sample, indexed by $l$, $l \in [1,n]$, is used at each training round $(k,s)$. Denote by $\mathbb{E}_{l}$ the expectation on one data sample. 
		
		To begin with, we introduce the following assumptions on the gradients and the aforementioned errors.
		
		\begin{assumption}
			\label{ass:M-bound}
			The $\ell_2$ norm of the difference between $\mathbf{v}^{(k,s)}$ and $\nabla f(\mathbf{x}^{(k-1,s)})$ is bounded by a constant, i.e.
			\vspace{-3pt}
			\begin{equation*}
				\mathbb{E}_l \Vert \mathbf{v}^{(k,s)} - \nabla f(\mathbf{x}^{(k-1,s)}) \Vert \le M  \qquad \forall k,s,l.
			\end{equation*}
		\end{assumption}
		
		\begin{assumption}
			\label{ass:summable}
			For any fixed $s$, $\mathbb{E} \Vert \mathbf{e}^{(k,s)} \Vert$ and $\sqrt{\varepsilon^{(k,s)}}$ are summable sequences with regard to $k$, i.e.,
			\vspace{-3pt}
			\begin{equation}
				\sum_{k=1}^{\infty} \mathbb{E} \Vert \mathbf{e}^{(k,s)} \Vert < \infty \quad ;\quad  \sum_{k=1}^{\infty} 	\sqrt{\varepsilon^{(k,s)}} < \infty.
			\end{equation}
		\end{assumption}
		
		Assumption \ref{ass:M-bound} bounds the error of the local gradient estimator $\mathbf{v}^{(k,s)}$ with a constant $M$. Assumption \ref{ass:summable} restricts the accumulated proximal error and the accumulated gradient error to be finite. Given the linkage between Inexact Prox-SVRG and DPSVRG, we prove the convergence of Inexact Prox-SVRG first. 
		
		\begin{theorem}
			\label{thm:inexact}
			If Inexact Prox-SVRG satisfies Assumptions $3\sim 6$, and $\alpha < \frac{\delta}{L(4\delta+8)}$ with $\delta<1$, then 
			\begin{equation*}
				\begin{aligned}
					\mathbb{E}[F(\tilde{\mathbf{x}}^S)-F(\mathbf{x}^*)] < &\rho^S \Big((\frac{\rho}{2n_0}+1)\mathbb{E}\big(F(\mathbf{x}^{init})-F(\mathbf{x}^*)\big)
					\\
					&+ \frac{\mathbb{E}\Vert \mathbf{x}^{init} - \mathbf{x}^*\Vert^2}{n_0(2\alpha-8\alpha^2L)} + \frac{C}{n_0(\rho\beta-1)}\Big)
				\end{aligned}
			\end{equation*}
			with a constant $C$ and $\rho = \frac{8\alpha L}{1 -4\alpha L}$. Here $\mathbf{x}^*$ refers to the optimal point of \eqref{P2} (or \eqref{P1}). Notation $\mathbb{E}$ denotes the expectation with all the data sampled in the training process.
		\end{theorem}
		
		\textbf{The Sketch of Proof.}	
		Due to the two-loop structure for variance reduction, the convergence analysis is carried out in two stages. The \emph{inner loop analysis} aims to show that the gap between the expected objective function and the optimum tends to decrease as we execute the inner loop iteration. Similarly, the \emph{outer loop analysis} demonstrates that this gap further decreases to zero as the number of outer loops approaches infinity. The detailed proof is presented in the supplementary.
		
		\begin{remark}
			As $\alpha < \frac{\delta}{L(4\delta+8)}$ and $\rho = \frac{8\alpha L}{1 -4\alpha L}$, we have $\rho < \delta < 1$. During $S$ outer iterations, the expected distance from the current objective value to the global optimum will decrease exponentially. However, if we take the total training rounds $T=\sum_{s=1}^S K_s$ into consideration, the convergence rate is at $O(\frac{1}{T})$, which has the same order with the inexact proximal gradient method utilizing the full gradient.
		\end{remark}
	
		Before presenting the main proof of Theorem \ref{thm:inexact}, we first show some necessary lemmas and leave their proofs in the supplementary file. First, we define the $\varepsilon$-subdifferential of convex function $h$ at $\mathbf{z}$ ($\partial_{\varepsilon}h(\mathbf{z})$) as the set of $\mathbf{p}$ satisfying 
		\begin{equation}
			\label{eq:subdifferential}
			h(\mathbf{y})\ge h(\mathbf{z})+\langle \mathbf{p},\mathbf{y}-\mathbf{z} \rangle-\varepsilon, \quad \forall{\mathbf{y}}.
		\end{equation}
		
		\begin{lemma}
			\label{lemma:g}
			Let $h$ be a proper closed and convex function and  $\mathbf{x} = {\rm prox}_{h,\varepsilon}^{\alpha} \{\mathbf{z}\}$. Then there exists $\mathbf{g}$ such that
			\begin{subequations}
				\label{eq:two equations}
				\begin{equation}
					\frac{1}{\alpha}(\mathbf{z} + \mathbf{g}- \mathbf{x}) \in \partial_{\varepsilon}h(\mathbf{x}) \label{eq:subgra-h},
				\end{equation}
				\begin{equation}
					\Vert \mathbf{g} \Vert \le \sqrt{2\alpha \varepsilon} \label{eq:g}.
				\end{equation}
			\end{subequations}
		\end{lemma}
		
		\begin{lemma}
			\label{lemma:inexact nonexpansive}
			Let $h$ be a proper closed and convex function, then for any two points $\mathbf{z}_1, \mathbf{z}_2 \in {\rm dom} \ h$, there exists 
			\begin{equation*}
				\Vert {\rm prox}_{h,\varepsilon}^\alpha \{\mathbf{z}_1\} - {\rm prox}_{h,\varepsilon}^\alpha \{\mathbf{z}_2\} \Vert \le \Vert \mathbf{z}_1-\mathbf{z}_2 \Vert + 3\sqrt{2\alpha \varepsilon}.
			\end{equation*}
		\end{lemma}

		\begin{lemma}
			\label{lemma:svrg}
			The variance of the gradient estimator $\mathbf{v}^{(k,s)}$ in Algorithm \ref{alg:inexact prox-svrg} can be bounded as the following
			\begin{equation*}
				\begin{aligned}
					\mathbb{E}_l \Vert \mathbf{v}^{(k,s)}-\nabla f(\mathbf{x}^{(k-1,s)}) \Vert^2 &\le 4L\left( F(\mathbf{x}^{(k-1,s)})-F(\mathbf{x}^*)\right) \\
					&+ 4L\left(F(\tilde{\mathbf{x}}^s)-F(\mathbf{x}^*)\right).
				\end{aligned}
			\end{equation*}
		\end{lemma}
		
		Lemma \ref{lemma:inexact nonexpansive} provides a relationship between the parameters before and after the proximal mapping. Lemma \ref{lemma:svrg} is utilized to bound the variance of the gradient estimator by the function values. Since both $\mathbf{x}^{(k-1,s)}$ and $\tilde{\mathbf{x}}^s$ are the estimators of $\mathbf{x}^*$, their function values will decrease during the training. Then, the variance of $\mathbf{v}^{(k,s)}$ tends to be zero as $k$ goes to infinity.
		
		\begin{lemma}
			\label{lemma:uk}
			Assume that the nonnegative sequence $\{u_K\}$ satisfies the following recursion constraint
			\vspace{-5pt}
			\begin{equation*}
				u_K^2 \le S_K + \sum_{k=1}^K \lambda_k u_k,
			\end{equation*}
			with $\{S_K\}$ an increasing sequence, $\lambda_k \ge 0 \ for \ \forall{\lambda_k}$ and $u_0^2 \le S_0$, then for all $k \ge 0$,
			\vspace{-5pt}
			\begin{equation}
				\label{eq:uK}
				u_K \le \frac{1}{2}\sum_{k=1}^K \lambda_k + \left(S_K + (\frac{1}{2}\sum_{k=1}^K \lambda_k)^2\right)^\frac{1}{2}.
			\end{equation}
		\end{lemma}
	
		\begin{proof}
				
			\textbf{Inner loop analysis}
			For the moment, we omit index $s$ in $(k,s)$ and substitute $K_s$ and $\tilde{\mathbf{x}}^{s-1}$ with $K$, and $\tilde{\mathbf{x}}$ respectively. We begin the proof with the update rules in Algorithm \ref{alg:inexact prox-svrg}. Applying Lemma \ref{lemma:g} to $\mathbf{x}^{(k)} = prox_{h, \varepsilon^{(k)}}^{\alpha}\{\mathbf{q}^{(k)}\}$ by defining 
			\begin{equation}
				\mathbf{w} = \frac{1}{\alpha}(\mathbf{q}^{(k)} + \mathbf{g} - \mathbf{x}^{(k)}) \in \partial_{\varepsilon^{(k)}} h(\mathbf{x}^{(k)}),
			\end{equation} 
			with $\Vert \mathbf{g} \Vert \le \sqrt{2\alpha \varepsilon^{(k)}}$ and introducing $\mathbf{q}^{(k)} = \mathbf{x}^{(k-1)}-\alpha (\mathbf{v}^{(k)}+\mathbf{e}^{(k)})$ yield
			\begin{equation}
				\label{eq:xi}
				\mathbf{x}^{(k-1)} - \mathbf{x}^{(k)} = \alpha (\mathbf{w} + \mathbf{v}^{(k)} + \mathbf{e}^{(k)}) - \mathbf{g} \triangleq \alpha \mathbf{\xi}^{(k-1)}.
			\end{equation}
			Hence, we have $\mathbf{\xi}^{(k-1)} = (\mathbf{w} + \mathbf{v}^{(k)} + \mathbf{e}^{(k)}) - \frac{1}{\alpha} \mathbf{g}$. We start the main analysis by evaluating the distance between the latest updated point and the optimum $\Vert \mathbf{x}^{(k)} - \mathbf{x}^* \Vert^2$, and bound it in the form of $\Vert \mathbf{x}^{(k-1)} - \mathbf{x}^*\Vert^2$ in the following formula.
			\begin{equation}
				\label{eq:xk-x*}
				\begin{aligned}
					&\Vert \mathbf{x}^{(k)}-\mathbf{x}^*\Vert^2 = \Vert \mathbf{x}^{(k-1)}-\alpha \mathbf{\xi}^{(k-1)} - \mathbf{x}^*\Vert^2 \\
					&= \Vert \mathbf{x}^{(k-1)} - \mathbf{x}^*\Vert^2 + \alpha^2\Vert\mathbf{\xi}^{(k-1)}\Vert^2 - 2\alpha \langle \mathbf{\xi}^{(k-1)}, \mathbf{x}^{(k-1)}-\mathbf{x}^* \rangle \\
					&= \Vert \mathbf{x}^{(k-1)} - \mathbf{x}^*\Vert^2 + \alpha^2\Vert\mathbf{\xi}^{(k-1)}\Vert^2 - 2\alpha\langle \mathbf{\xi}^{(k-1)}, \mathbf{x}^{(k-1)}-\mathbf{x}^{(k)} \rangle \\
					&\quad - 2\alpha\langle \mathbf{\xi}^{(k-1)}, \mathbf{x}^{(k)}-\mathbf{x}^* \rangle \\
					&\overset{\eqref{eq:xi}}{=} \Vert \mathbf{x}^{(k-1)} - \mathbf{x}^*\Vert^2 - \underbrace{\alpha^2 \Vert\mathbf{\xi}^{(k-1)}\Vert^2 - 2\alpha \langle \mathbf{\xi}^{(k-1)}, \mathbf{x}^{(k)}-\mathbf{x}^* \rangle}
				\end{aligned}
			\end{equation}
			The next proposition shows an upper bound on $-\alpha^2 \Vert\mathbf{\xi}^{(k-1)}\Vert^2 - 2\alpha \langle \mathbf{\xi}^{(k-1)}, \mathbf{x}^{(k)}-\mathbf{x}^* \rangle$.
			
			\begin{proposition}
				\vspace*{0pt}
				\label{prop:xik-1}
				With Assumptions \ref{ass:f-property}, \ref{ass:M-bound} and $\alpha L \le 1$, we have
				\vspace{-5pt}
				\begin{align}
					&-\alpha^2 \Vert\mathbf{\xi}^{(k-1)}\Vert^2 - 2\alpha \langle \mathbf{\xi}^{(k-1)}, \mathbf{x}^{(k)}-\mathbf{x}^* \rangle \notag\\
					&\le - 2\alpha (F(\mathbf{x}^{(k)}) - F(\mathbf{x}^*)) + 2\alpha \varepsilon^{(k)} + 2\alpha^2\Vert \mathbf{v}^{(k)} - \nabla f(\mathbf{x}^{(k-1)})\Vert^2 \notag \\
					&\quad + (2\alpha^2\Vert \mathbf{e}^{(k)}\Vert+6\alpha \sqrt{2\alpha\varepsilon^{(k)}})M - 2\alpha \langle \mathbf{x}^{(k)}-\mathbf{x}^*, \mathbf{e}^{(k)} - \frac{1}{\alpha}\mathbf{g} \rangle \notag\\
					&\quad - 2\alpha \langle \underline{\mathbf{x}}-\mathbf{x}^*, \mathbf{v}^{(k)} - \nabla f(\mathbf{x}^{(k-1)}) \rangle \notag.
				\end{align}
			\end{proposition}
			\vspace*{-5pt}
			
			The proof of proposition \ref{prop:xik-1} is shown in the supplementary file. We apply this proposition to \eqref{eq:xk-x*} and take the expectation on both sides of \eqref{eq:xk-x*} w.r.t. sample batch $l_{in}$. Notice that the expectation $\mathbb{E}_{l_{in}} [- 2\alpha \langle \underline{\mathbf{x}}-\mathbf{x}^*, \mathbf{v}^{(k)} - \nabla f(\mathbf{x}^{(k-1)}) \rangle] = - 2\alpha \langle \underline{\mathbf{x}}-\mathbf{x}^*, \mathbb{E}_{l_{in}}[\mathbf{v}^{(k)} - \nabla f(\mathbf{x}^{(k-1)})] \rangle = 0$, since $(\underline{\mathbf{x}}-\mathbf{x}^*)$ is deterministic w.r.t. $l_{in}$.
			\begin{align}
				\label{eq:Exk-x*}
				&\mathbb{E}_{l_{in}}\Vert \mathbf{x}^{(k)}-\mathbf{x}^*\Vert^2 \notag\\
				&\le \Vert \mathbf{x}^{(k-1)} - \mathbf{x}^*\Vert^2 - 2\alpha \mathbb{E}_{l_{in}}(F(\mathbf{x}^{(k)}) - F(\mathbf{x}^*)) \notag \\
				& \quad + 2\alpha^2\mathbb{E}_{l_{in}}\Vert \mathbf{v}^{(k)} - \nabla f(\mathbf{x}^{(k-1)})\Vert^2 \notag \\
				&\quad + (2\alpha^2\mathbb{E}_{l_{in}}\Vert \mathbf{e}^{(k)}\Vert+6\alpha \sqrt{2\alpha\varepsilon^{(k)}})M \notag\\
				&\quad + 2\alpha \mathbb{E}_{l_{in}}\left[\Vert \mathbf{x}^{(k)}-\mathbf{x}^*\Vert\Big(\Vert \mathbf{e}^{(k)}\Vert+\frac{1}{\alpha}\Vert \mathbf{g} \Vert\Big)\right] + 2\alpha \varepsilon^{(k)} \notag\\
				&\le \Vert \mathbf{x}^{(k-1)} - \mathbf{x}^*\Vert^2 - 2\alpha \mathbb{E}_{l_{in}}(F(\mathbf{x}^{(k)}) - F(\mathbf{x}^*)) \notag\\
				&\quad+ 8\alpha^2L\big(F(\mathbf{x}^{(k-1)})-F(\mathbf{x}^*)\big) + 8\alpha^2L\big(F(\tilde{\mathbf{x}})-F(\mathbf{x}^*)\big) \notag\\
				&\quad+ (2\alpha^2\mathbb{E}_{l_{in}}\Vert \mathbf{e}^{(k)}\Vert+6\alpha \sqrt{2\alpha\varepsilon^{(k)}})M \notag\\
				&\quad+ 2\alpha \mathbb{E}_{l_{in}}\Vert \mathbf{x}^{(k)}-\mathbf{x}^*\Vert\Big(\mathbb{E}_{l_{in}}\Vert \mathbf{e}^{(k)}\Vert+\frac{1}{\alpha}\Vert \mathbf{g} \Vert\Big) + 2\alpha \varepsilon^{(k)}.
				\vspace{-5pt}
			\end{align}
			
			The inequality utilizes Lemma \ref{lemma:svrg} for $\mathbb{E}_{l_{in}}\Vert \mathbf{v}^{(k)} - \nabla f(\mathbf{x}^{(k-1)})\Vert^2$. Summing \eqref{eq:Exk-x*} over $k$ from $1$ to $K$ and taking the expectation w.r.t. all data in $K$ iterations bring forth
			\vspace*{-5pt}
			\begin{align}
				\label{eq:sumE}
				&2\alpha \sum_{k=1}^K \mathbb{E}\big(F(\mathbf{x}^{(k)}) - F(\mathbf{x}^*)\big) \le \Vert \mathbf{x}^{(0)} - \mathbf{x}^*\Vert^2 - \mathbb{E}\Vert \mathbf{x}^{(K)}-\mathbf{x}^*\Vert^2 \notag\\
				&+ 8\alpha^2L\sum_{k=1}^K\mathbb{E}\big(F(\mathbf{x}^{(k-1)})-F(\mathbf{x}^*)\big) +8\alpha^2LK\big(F(\tilde{\mathbf{x}})-F(\mathbf{x}^*)\big) \notag\\
				&+ 2\alpha^2M\sum_{k=1}^K\mathbb{E}\Vert \mathbf{e}^{(k)}\Vert + 6\alpha M \sqrt{2\alpha}\sum_{k=1}^K\sqrt{\varepsilon^{(k)}} \notag\\
				&+\sum_{k=1}^K\left(2\alpha\mathbb{E}\Vert \mathbf{e}^{(k)}\Vert+2\sqrt{2\alpha\varepsilon^{(k)}}\right)\mathbb{E}\Vert \mathbf{x}^{(k)}-\mathbf{x}^*\Vert+ 2\alpha \sum_{k=1}^K\varepsilon^{(k)},
				\vspace*{-5pt}
			\end{align}
			where we apply \eqref{eq:g} to enlarge $\Vert \mathbf{g} \Vert$. With the rearrangement of \eqref{eq:sumE} and assumption $\alpha <\frac{1}{4L}$, we have
			\vspace*{-5pt}
			\begin{align}
				\label{eq:sum}
				&(2\alpha-8\alpha^2L) \sum_{k=1}^{K} \mathbb{E} \big(F(\mathbf{x}^{(k)}) - F(\mathbf{x}^*)\big)+ \mathbb{E} \Vert \mathbf{x}^{(K)}-\mathbf{x}^*\Vert^2 \notag\\
				&\quad + 8\alpha^2L\mathbb{E} \big(F(\mathbf{x}^{(K)})-F(\mathbf{x}^*)\big)  \notag\\
				&\le  \Vert \mathbf{x}^{(0)} - \mathbf{x}^*\Vert^2 + 8\alpha^2L\big(F(\mathbf{x}^{(0)})-F(\mathbf{x}^*)\big)
				\notag\\
				&\quad + 8\alpha^2LK\big(F(\tilde{\mathbf{x}})-F(\mathbf{x}^*)\big) + 2\alpha^2M\sum_{k=1}^K\mathbb{E}\Vert \mathbf{e}^{(k)}\Vert \notag\\
				&\quad + 6\alpha M \sqrt{2\alpha}\sum_{k=1}^K\sqrt{\varepsilon^{(k)}} +  2\alpha \sum_{k=1}^K\varepsilon^{(k)} \notag\\
				&\quad +\sum_{k=1}^K\left(2\alpha\mathbb{E}\Vert \mathbf{e}^{(k)}\Vert+2\sqrt{2\alpha\varepsilon^{(k)}}\right)\mathbb{E}\Vert \mathbf{x}^{(k)}-\mathbf{x}^*\Vert.
			\end{align}

			We temporarily only keep $\mathbb{E} \Vert \mathbf{x}^{(K)}-\mathbf{x}^*\Vert^2$ on the left side of \eqref{eq:sum} due to the positiveness of other two terms and obtain
			\vspace*{-5pt}
			\begin{equation}
				\label{eq:Exk-x*}
				\begin{aligned}
					\mathbb{E} &\Vert \mathbf{x}^{(K)}-\mathbf{x}^*\Vert^2 \le \Vert \mathbf{x}^{(0)} - \mathbf{x}^*\Vert^2 + 8\alpha^2L\big(F(\mathbf{x}^{(0)})-F(\mathbf{x}^*)\big)
					\\
					&+ 8\alpha^2LK\big(F(\tilde{\mathbf{x}})-F(\mathbf{x}^*)\big) + 2\alpha^2M\sum_{k=1}^K\mathbb{E}\Vert \mathbf{e}^{(k)}\Vert\\
					&+ 6\alpha M \sqrt{2\alpha}\sum_{k=1}^K\sqrt{\varepsilon^{(k)}} +  2\alpha \sum_{k=1}^K\varepsilon^{(k)}\\
					&+\sum_{k=1}^K\left(2\alpha\mathbb{E}\Vert \mathbf{e}^{(k)}\Vert+2\sqrt{2\alpha\varepsilon^{(k)}}\right)\mathbb{E}\Vert \mathbf{x}^{(k)}-\mathbf{x}^*\Vert.
				\end{aligned}
			\end{equation}
			
			Next, with this recursion model between $\mathbb{E}\Vert \mathbf{x}^{(K)}-\mathbf{x}^*\Vert^2$ and $\mathbb{E}\Vert \mathbf{x}^{(k)}-\mathbf{x}^*\Vert$, we can set up a bound for $\mathbb{E}\Vert \mathbf{x}^{(k)} - \mathbf{x}^* \Vert$ which only contains the objective function values, the squared $\ell$-2 norm of $\mathbf{x}^{(k)}$, $\mathbf{x}^{(0)}$ and the error terms. To be specific, using Lemma \ref{lemma:uk} with $\lambda_k = 2\alpha\mathbb{E}\Vert \mathbf{e}^{(k)}\Vert+2\sqrt{2\alpha\varepsilon^{(k)}}$ and $u_k = \mathbb{E}\Vert \mathbf{x}^{(k)}-\mathbf{x}^*\Vert$ and taking the average of \eqref{eq:Exk-x*} yield the next proposition.
			\begin{proposition}
				\label{prop:inner-final}
				Let $\theta = 2 \alpha - 8 \alpha^2 L$, under Assumption \ref{ass:summable}, we have
				\begin{align}
					\label{eq:inner-last}
					&\frac{1}{K}\sum_{k=1}^K\mathbb{E}\big(F(\mathbf{x}^{(k)}) - F(\mathbf{x}^*)\big) + \frac{\mathbb{E}\Vert \mathbf{x}^{(K)}-\mathbf{x}^*\Vert^2}{K\theta} \notag\\
					&\quad +\frac{8\alpha^2L}{K\theta}\mathbb{E}\big(F(\mathbf{x}^{(K)})-F(\mathbf{x}^*)\big)\notag\\
					&\le \frac{16\alpha^2L}{\theta} \big(F(\tilde{\mathbf{x}})-F(\mathbf{x}^*)\big) + \frac{2\Vert \mathbf{x}^{(0)} - \mathbf{x}^*\Vert^2}{K\theta} \notag\\
					&\quad +\frac{16\alpha^2L}{K\theta} \big(F(\mathbf{x}^{(0)})-F(\mathbf{x}^*)\big) + \frac{C}{K},
				\end{align}
				with $C$ is constant bounding the sum $\Big(3A_K^2 + \big(\sqrt{2\alpha}B_K+A_K\big)^2 + \big(\sqrt{6\alpha M \sqrt{2\alpha} B_K}+A_K\big)^2 + \big(\sqrt{2\alpha^2MC_K}+A_K\big)^2\Big)/\theta$, where $A_K = \sum_{k=1}^K\big(\alpha\mathbb{E}\Vert \mathbf{e}^{(k)}\Vert +\sqrt{2\alpha\varepsilon^{(k)}}\big)$, $B_K = \sum_{k=1}^K\sqrt{\varepsilon^{(k)}}$ and $C_K = \sum_{k=1}^K\mathbb{E}\Vert \mathbf{e}^{(k)}\Vert$.
			\end{proposition}
			
			Proposition \ref{prop:inner-final} presents a rough relationship between the function values of the last state ($\mathbf{x}^{(K)}$) and that of the initial state ($\mathbf{x}^{(0)}$) in one inner loop. Its detailed proof is provided in the supplementary file.

			\textbf{Outer loop analysis}
			In the following, we consider the outer loop and reuse $s$ in \eqref{eq:inner-last}. Owing to the convexity of $F(x)$, we have $\frac{1}{K_s}\sum_{k=1}^{K_s}\mathbb{E}F(\mathbf{x}^{(k,s)}) \ge \mathbb{E}\left[\frac{1}{K_s}\sum_{k=1}^{K_s} F(\mathbf{x}^{(k,s)})\right] \ge \mathbb{E}\left[F\big(\frac{1}{K_s}\sum_{k=1}^{K_s} \mathbf{x}^{(k,s)}\big)\right] = \mathbb{E}F(\tilde{\mathbf{x}}^s)$. Therefore, taking the expectation w.r.t all sampled data in $s$ outer rounds, \eqref{eq:inner-last} becomes
			\vspace*{-5pt}
			\begin{align}
				\label{eq:outer-begin}
				&\mathbb{E}\left(F(\tilde{\mathbf{x}}^s)-F(\mathbf{x}^*)\right) + \frac{\mathbb{E}\Vert \mathbf{x}^{(K_s,s)}-\mathbf{x}^*\Vert^2}{K_s\theta} \notag\\
				&\quad +\frac{8\alpha^2L}{K_s\theta}\mathbb{E}\big(F(\mathbf{x}^{(K_s,s)})-F(\mathbf{x}^*)\big)  \notag\\
				&\le \frac{16\alpha^2L}{\theta} \mathbb{E}\big(F(\tilde{\mathbf{x}}^{s-1})-F(\mathbf{x}^*)\big) + \frac{2\mathbb{E}\Vert \mathbf{x}^{(0,s)} - \mathbf{x}^*\Vert^2}{K_s\theta} \notag\\
				&\quad +\frac{16\alpha^2L}{K_s\theta} \mathbb{E}\big(F(\mathbf{x}^{(0,s)})-F(\mathbf{x}^*)\big)+ \frac{C}{K_s}.
			\end{align}
			Equation \eqref{eq:outer-begin} is a recursion model with variable $s$. Telescoping this recursion from $s=1$ to $S$ gives birth to
			\vspace*{-5pt}
			\begin{align}
				\label{eq:Efxtilde}
				&\mathbb{E}\left(F(\tilde{\mathbf{x}}^s)-F(\mathbf{x}^*)\right) + \frac{\mathbb{E}\Vert \mathbf{x}^{(0,s+1)}-\mathbf{x}^*\Vert^2}{K_s\theta}\notag\\
				&\quad +\frac{\rho}{2K_s}\mathbb{E}\big(F(\mathbf{x}^{(0,s+1)})-F(\mathbf{x}^*)\big) \notag \\
				&\le \rho \Big( \mathbb{E}\big(F(\tilde{\mathbf{x}}^{s-1})-F(\mathbf{x}^*)\big) + \frac{\mathbb{E}\Vert \mathbf{x}^{(0,s)} - \mathbf{x}^*\Vert^2}{K_{s-1}\theta} \notag\\
				&\quad +\frac{\rho}{2 K_{s-1}} \mathbb{E}\big(F(\mathbf{x}^{(0,s)})-F(\mathbf{x}^*)\big) \Big)+ \frac{C}{K_s},
			\end{align}
			where $\rho = \frac{16\alpha^2L}{\theta} < \delta <1$ with the constant $\delta$. The inequality is due to the condition $\rho \beta \ge 2$. Telescoping \eqref{eq:Efxtilde} over $s$ from $1$ to $S$, we obtain
			\vspace*{-5pt}
			\begin{align}
				&\mathbb{E}\left(F(\tilde{\mathbf{x}}^S)-F(\mathbf{x}^*)\right) + \frac{\mathbb{E}\Vert \mathbf{x}^{(0,S+1)}-\mathbf{x}^*\Vert^2}{K_S\theta} \notag\\
				&\quad +\frac{\rho}{2K_S}\mathbb{E}\big(F(\mathbf{x}^{(0,S+1)})-F(\mathbf{x}^*)\big) \notag\\
				&< \rho^S \Big(\mathbb{E}\big(F(\tilde{\mathbf{x}}^{0})-F(\mathbf{x}^*)\big) + \frac{\mathbb{E}\Vert \mathbf{x}^{(0,1)} - \mathbf{x}^*\Vert^2}{K_0\theta} \notag\\
				&\quad +\frac{\rho}{2K_0} \mathbb{E}\big(F(\mathbf{x}^{(0,1)})-F(\mathbf{x}^*)\big) \Big) + \frac{C}{n_0} \frac{\rho^S}{\rho \beta-1},
			\end{align}
			where we apply $\sum_{i=0}^{S-1}\frac{\rho^i}{\beta^{S-i}} = \frac{\sum_{i=1}^{S-1}(\rho\beta)^i}{\beta^S} = \frac{1}{\beta^S}\frac{1-(\rho\beta)^S}{1-\rho\beta} = \frac{\rho^S-\frac{1}{\beta^S}}{\rho\beta-1} < \frac{\rho^S}{\rho\beta-1}$ in the inequality. We complete the proof of Theorem \ref{thm:inexact} by using $\tilde{\mathbf{x}}^0 = \mathbf{x}^{(0,1)} = \mathbf{x}^{init}$ and removing $\frac{\mathbb{E}\Vert \mathbf{x}^{(0,S+1)}-\mathbf{x}^*\Vert^2}{K_S\theta} +\frac{\rho}{2K_S}\mathbb{E}\big(F(\mathbf{x}^{(0,S+1)})-F(\mathbf{x}^*)\big) $ due to their positiveness. 
			
			Note that, in the above analysis, $\alpha$ needs to satisfy the following condition: $\alpha < \min \{\frac{\delta}{L(4\delta+8)}, \frac{1}{4L}, \frac{1}{L}\} = \frac{\delta}{L(4\delta+8)}$, where the first condition $\alpha < \frac{\delta}{L(4\delta+8)}$ comes from $\rho = \frac{16\alpha^2L}{\theta} < \delta$.
		\end{proof}

		\subsection{Convergence of DPSVRG}
		\label{subsec:de-proxsvrg}
		
		DPSVRG is reducible to Inexact Prox-SVRG when the gradient and the proximal mapping errors of Inexact Prox-SVRG are carefully chosen (as shown in Theorem \ref{thm:transform}). The remaining task is to examine whether DPSVRG satisfies the conditions in Theorem \ref{thm:inexact} so that it can achieve the same convergence rate with that of Inexact Prox-SVRG. Specifically, we prove that under the following assumption, the summability of error terms in \eqref{eq:e and epsilon} is validated to be in line with Assumption \ref{ass:summable}.
		
		\begin{assumption}
			\label{ass:gradient bound}
			For each node $i$ in (\ref{P1}), the $\ell$-2 norm of the gradients of $f_i(\mathbf{x})$ and the subgradients of $h(\mathbf{x})$ are bounded by certain constants,
			\begin{equation}
				\Vert \nabla f_i(\mathbf{x}_i) \Vert \le G_f, \quad \forall i\in \{1,...,m\},
			\end{equation}
			\begin{equation}
				\label{eq:self9}
				\Vert \mathbf{z} \Vert \le G_h, \quad \forall z\in \partial h(\mathbf{x}_i).
			\end{equation}
		\end{assumption}

		This assumption confines the gradients of $f$ for each node and the subgradients of $h$. Now we state the convergence of Distributed Prox-SVRG in Theorem \ref{thm:prox-svrg}.

		\begin{theorem}
			\label{thm:prox-svrg}
			With assumptions \ref{ass:connectivity} $\sim$ \ref{ass:M-bound} and \ref{ass:gradient bound}, DPSVRG achieves the identical convergence rate with Inexact Prox-SVRG.
		\end{theorem}
		
		\begin{proof}
			Taking the $\ell_2$ norm and the expectation of both sides in \eqref{eq:e} yields
			\vspace*{-5pt}
			\begin{equation}
				\label{eq:Eeks}
				\begin{aligned}
					&\mathbb{E} \Vert \mathbf{e}^{(k,s)} \Vert
					\le \frac{1}{m}\sum_{i=1}^m \Big(\underbrace{\mathbb{E}\Vert\nabla f_i^{l_i}(\mathbf{x}_i^{(k-1, s)}) - \nabla f_i^{l_i}(\bar{\mathbf{x}}^{(k-1,s)})\Vert} +\\
					&\mathbb{E}\Vert \nabla f_i^{l_i}(\tilde{\mathbf{x}}^{s-1}) - \nabla f_i^{l_i}(\tilde{\mathbf{x}}_i^{s-1})\Vert + \mathbb{E}\Vert\nabla f_i(\tilde{\mathbf{x}}_i^{s-1}) - \nabla f_i(\tilde{\mathbf{x}}^{s-1})\Vert \Big).
				\end{aligned}
			\end{equation}
			
			Now we process the first term on the right and obtain
			\begin{equation}
				\label{eq:Eli}
				\begin{aligned}
					&\mathbb{E}\Vert\nabla f_i^{l_i}(\mathbf{x}_i^{(k-1, s)}) - \nabla f_i^{l_i}(\bar{\mathbf{x}}^{(k-1,s)})\Vert \\
					&= \frac{1}{n_i} \sum_{j=1}^{n_i} \Vert\nabla f_i^{j}(\mathbf{x}_i^{(k-1, s)}) - \nabla f_i^{j}(\bar{\mathbf{x}}^{(k-1,s)})\Vert \\
					&\le L \Vert \mathbf{x}_i^{(k-1, s)} - \bar{\mathbf{x}}^{(k-1,s)} \Vert.
				\end{aligned}
			\end{equation}
			where the inequality applies the smoothness of $f$ in Assumption \ref{ass:f-property}. The second and the third terms in \eqref{eq:Eeks} can be bounded with the same technique in \eqref{eq:Eli} as follows,
			\begin{subequations}
				\begin{equation}
					\mathbb{E}\Vert \nabla f_i^{l_i}(\tilde{\mathbf{x}}^{s-1}) - \nabla f_i^{l_i}(\tilde{\mathbf{x}}_i^{s-1})\Vert \le L \Vert \tilde{\mathbf{x}}^{s-1} - \tilde{\mathbf{x}}_i^{s-1} \Vert,
				\end{equation}
				\begin{equation}
					\mathbb{E}\Vert\nabla f_i(\tilde{\mathbf{x}}_i^{s-1}) - \nabla f_i(\tilde{\mathbf{x}}^{s-1})\Vert \le L \Vert \tilde{\mathbf{x}}^{s-1} - \tilde{\mathbf{x}}_i^{s-1} \Vert.
				\end{equation}
			\end{subequations}
			Thus we obtain
			\begin{equation}
				\begin{aligned}
					\mathbb{E} \Vert \mathbf{e}^{(k,s)} \Vert &\le \frac{L}{m}\sum_{i=1}^m \Big(\underbrace{\Vert \mathbf{x}_i^{(k-1, s)} - \bar{\mathbf{x}}^{(k-1,s)}\Vert} + 2\Vert\tilde{\mathbf{x}}_i^{s-1} - \tilde{\mathbf{x}}^{s-1}\Vert\Big).
				\end{aligned}
			\end{equation}
			The first term is bounded in the following way.
			\begin{align}
				\label{eq:LM}
				&\frac{L}{m}\sum_{i=1}^m\Vert \mathbf{x}_i^{(k-1, s)} - \bar{\mathbf{x}}^{(k-1,s)}\Vert \le \frac{L}{m^2}\sum_{i=1}^m\sum_{j=1}^m\Vert \mathbf{x}_i^{(k-1,s)} - \mathbf{x}_j^{(k-1,s)}\Vert \notag\\
				&\le \frac{L}{m^2}\sum_{i=1}^m\sum_{j=1}^m\Vert \hat{\mathbf{q}_i}^{(k-1,s)} - \hat{\mathbf{q}_j}^{(k-1,s)}\Vert \notag\\
				&\le \frac{2L}{m}\sum_{i=1}^m\Vert\hat{\mathbf{q}_i}^{(k-1,s)} - \bar{\mathbf{q}}^{(k-1,s)}\Vert \notag\\
				&\le \frac{2L}{m}\sum_{i=1}^m (\Gamma \gamma^{k} \sum_{j=1}^m\Vert \mathbf{q}_j^{(k-1,s)}\Vert) = 2L \Gamma \gamma^{k} \sum_{i=1}^m\Vert \mathbf{q}_i^{(k-1,s)}\Vert,
			\end{align}
			where the second inequality uses nonexpansiveness of proximal operator (Lemma \ref{lemma:nonexpansive}) and the last inequality uses Lemma \ref{lemma:matrix converge} with the inequalities
			$\Vert \hat{\mathbf{q}_i}^{(k,s)} - \bar{\mathbf{q}}^{(k,s)}\Vert = \left\Vert\sum_{j=1}^m \phi_{i,j}^{(k,s)}\mathbf{q}_j^{(k,s)} - \frac{1}{m} \mathbf{q}_j^{(k,s)} \right\Vert
			\le \sum_{j=1}^m \Big|\phi_{i,j}^{(k,s)}-\frac{1}{m}\Big| \cdot \Vert \mathbf{q}_j^{(k,s)}\Vert
			\le \Gamma \gamma^k \sum_{j=1}^m \Vert \mathbf{q}_j^{(k,s)}\Vert.$
			
			Similarly, we give rise to
			\begin{align}
				&\frac{2L}{m}\sum_{i=1}^m \Vert \tilde{\mathbf{x}}_i^{s-1} - \tilde{\mathbf{x}}^{s-1}\Vert \le \frac{2L}{m^2}\sum_{i=1}^m \sum_{j=1}^m \Vert\tilde{\mathbf{x}}_i^{s-1}- \tilde{\mathbf{x}}_j^{s-1}\Vert \notag\\
				&= \frac{2L}{m^2}\sum_{i=1}^m \sum_{j=1}^m \Vert\frac{1}{K_s}\sum_{k=1}^{K_s} \mathbf{x}_i^{(k,s-1)}- \frac{1}{K_s}\sum_{k=1}^{K_s} \mathbf{x}_j^{(k,s-1)}\Vert \notag\\
				&\le \frac{2L}{m^2}\sum_{i=1}^m \sum_{j=1}^m \frac{1}{K_s}\sum_{k=1}^{K_s}\Vert \mathbf{x}_i^{(k,s-1)}- \mathbf{x}_j^{(k,s-1)}\Vert \notag\\
				&= 4L\Gamma \frac{1}{{K_s}}\sum_{k=1}^{K_s} \gamma^k \sum_{i=1}^m \Vert \mathbf{q}_i^{k,s-1}\Vert.
			\end{align}
			The third inequality utilizes the similar technique in \eqref{eq:LM}. Finally, we bound $\mathbb{E} \Vert \mathbf{e}^{(k,s)} \Vert$ with the sum of $\Vert \mathbf{q}_i^{(k,s)} \Vert$, i.e.,
			\vspace{-8pt}
			\begin{equation}
				\label{eq:Eeks-1}
				\begin{aligned}
					\mathbb{E} \Vert \mathbf{e}^{(k,s)} \Vert \le &2L \Gamma \gamma^{k} \sum_{i=1}^m\Vert \mathbf{q}_i^{(k-1,s)}\Vert + \frac{4L\Gamma}{{K_s}}\sum_{k=1}^{K_s} \gamma^k \sum_{i=1}^m \Vert \mathbf{q}_i^{k,s-1}\Vert.
				\end{aligned}
			\end{equation}
			Subsequently, we bound $\sqrt{\varepsilon}^{(k,s)}$ from \eqref{eq:varepsilon} using the similar techniques applied in disposing $\mathbf{e}^{(k,s)}$. With $\mathbf{y}^{(k,s)} = {\rm prox}_h^{\alpha} \{\bar{\mathbf{q}}^{(k,s)}\}$, we deal with the first term in \eqref{eq:varepsilon} in the following way.
			\vspace{-13pt}
			\begin{align}
				\label{eq:xbar-y}
				&\frac{1}{2\alpha}\Vert\bar{\mathbf{x}}^{(k,s)}-\mathbf{y}^{(k,s)}\Vert^2 = \frac{1}{2\alpha}\left\Vert\frac{1}{m} \sum_{i=1}^m \mathbf{x}_i^{(k,s)}-\frac{1}{m} \sum_{i=1}^m \mathbf{y}^{(k,s)}\right\Vert^2 \notag\\
				&\le \frac{1}{2\alpha}\left(\frac{1}{m} \sum_{i=1}^m \Vert \mathbf{x}_i^{(k,s)}-\mathbf{y}^{(k,s)}\Vert\right)^2 \notag\\
				&\le \frac{1}{2\alpha}\left(\frac{1}{m} \sum_{i=1}^m \Vert\hat{\mathbf{q}_i}^{(k,s)}-\bar{\mathbf{q}}^{(k,s)}\Vert \right)^2 \notag\\
				&\le \frac{1}{2\alpha}\left(\frac{1}{m} \sum_{i=1}^m \Gamma \gamma^{k} \sum_{j=1}^m \Vert \mathbf{q}_j^{(k,s)}\Vert\right)^2 = \frac{1}{2\alpha}\left(\Gamma \gamma^{k} \sum_{i=1}^m \Vert \mathbf{q}_i^{(k,s)}\Vert\right)^2.
			\end{align}
			The second term of \eqref{eq:varepsilon} is handled in the following way.
			\vspace{-10pt}
			\begin{align}
				\label{eq:xbar-z,p}
				&\langle \bar{\mathbf{x}}^{(k,s)}-\mathbf{y}^{(k,s)}, \frac{1}{\alpha}(\mathbf{y}^{(k,s)}-\bar{\mathbf{q}}^{(k,s)})+\mathbf{p} \rangle \notag\\
				&\le \Vert\bar{\mathbf{x}}^{(k,s)}-\mathbf{y}^{(k,s)}\Vert\cdot\Vert\frac{1}{\alpha}(\mathbf{y}^{(k,s)}-\bar{\mathbf{q}}^{(k,s)})+\mathbf{p}\Vert \notag\\
				&\le 2G_h\Vert\bar{\mathbf{x}}^{(k,s)}-\mathbf{y}^{(k,s)}\Vert \notag\\
				&\le 2G_h \frac{1}{m} \sum_{i=1}^m \Vert\hat{\mathbf{q}_i}^{(k,s)}-\bar{\mathbf{q}}^{(k,s)}\Vert \notag\\
				&\le 2G_h \frac{1}{m} \sum_{i=1}^m \Gamma \gamma^{k} \sum_{j=1}^m \Vert \mathbf{q}_j^{(k,s)}\Vert = 2G_h \Gamma \gamma^{k} \sum_{i=1}^m \Vert \mathbf{q}_i^{(k,s)}\Vert,
			\end{align}
			where the second inequality comes from $\Vert \bar{\mathbf{q}}^{(k,s)}-\mathbf{y}^{(k,s)}\Vert \le \alpha G_h$ by applying Assumption \ref{ass:gradient bound}. Combining \eqref{eq:xbar-y} and \eqref{eq:xbar-z,p}, $\sqrt{\varepsilon}^{(k,s)}$ is bounded by
			\begin{align}
				\label{eq:sqrtepsilon}
				&\sqrt{\varepsilon^{(k,s)}} \le \sqrt{\frac{1}{2\alpha}\left(\Gamma \gamma^{k} \sum_{i=1}^m \Vert \mathbf{q}_i^{(k,s)}\Vert\right)^2 + 2G_h \Gamma \gamma^{k} \sum_{i=1}^m \Vert \mathbf{q}_i^{(k,s)}\Vert} \notag\\
				&\le \frac{1}{2\alpha}\Gamma \gamma^{k} \sum_{i=1}^m \Vert \mathbf{q}_i^{(k,s)}\Vert + \sqrt{2G_h \Gamma \gamma^{k} \sum_{i=1}^m \Vert \mathbf{q}_i^{(k,s)}\Vert}.
			\end{align}
			
			In what follows, we will show that $\mathbb{E} \Vert \mathbf{e}^{(k,s)} \Vert$ and $\sqrt{\varepsilon^{(k,s)}}$ satisfy Assumption \ref{ass:summable}, i.e., they are summable over $k$. By introducing the following proposition \ref{prop:polynomial}, two error sequences can be bounded by terms related to $C_0$, $C_1$ and $C_2$, and thus their summations over $k$ are finite. 
			\begin{proposition}
				\vspace*{5pt}
				\label{prop:polynomial}
				Consider Algorithm \ref{alg:prox-svrg}. Let $C_0 = \sum_{i=1}^m \mathbf{q}_i^{(1,1)}< \infty$, $C_1 = \alpha m(G_g+G_h)< \infty$, $C_2 = \alpha m \beta^s n_0(G_g+G_h)$. Then, $\forall k,s$, the term $\sum_{i=1}^m \Vert \mathbf{q}_i^{(k,s)}\Vert$ can be bounded by $C_0 + C_1k + C_2s, \ (C_0, C_1, C_2 \ge 0)$, i.e.
				\begin{equation}
					\label{eq:polynomial}
					\sum_{i=1}^m \Vert \mathbf{q}_i^{(k,s)}\Vert \le C_0 + C_1k + C_2s.
				\end{equation}
			\end{proposition}
		
			Applying Proposition \ref{prop:polynomial} to \eqref{eq:Eeks-1} and \eqref{eq:sqrtepsilon}, we have
		\begin{subequations}
			\begin{align}
				\label{eq:eks-final}
				\sum_{k=1}^{K_s}\mathbb{E} &\Vert \mathbf{e}^{(k,s)} \Vert \le \sum_{k=1}^{K_s} 2L \Gamma \gamma^{k} \left(C_0 + C_1(k-1) + C_s s\right) \notag \\
				&+ \sum_{k=1}^{K_s} 4L\Gamma \frac{1}{{K_s}}\sum_{k=1}^{K_s} \gamma^k \left(C_0 + C_1k + C_s (s-1)\right) \notag\\
				\le &2L \Gamma \left( D_0^2 (C_0 - C_1 + C_s s) + C_1 D_1^2 \right) \notag\\
				& + 4L\Gamma \left(D_0^2 (C_0 + C_s (s-1)) + C_1 D_1^2 \right),
			\end{align}
			and
			\begin{align}
				\label{eq:epsilon-final}
				\sum_{k=1}^{K_s} \sqrt{\varepsilon^{(k,s)}} \le &\sum_{k=1}^{K_s} \frac{1}{2\alpha}\Gamma \gamma^{k} ( C_0 + C_1 k + C_s s) \notag\\
				&+ \sum_{k=1}^{K_s} \sqrt{2G_h \Gamma \gamma^{k} ( C_0 + C_1 k + C_s s)} \notag\\
				\le &\frac{\Gamma}{2\alpha} \left( D_0 (C_0 + C_2 s) + C_1 D_1 \right) \notag\\
				&+ \sqrt{2G_h \Gamma} \left( D_0 \sqrt{C_0 + C_2 s} + \sqrt{C_1} D_1 \right),
			\end{align}
		\end{subequations}
		with some constants $D_0 \ge \sum_{k=1}^{\infty} \sqrt{\gamma^k}$ and $D_1 \ge \sum_{k=1}^{\infty} \sqrt{\gamma^k k}$. Then there exists $\sum_{k=1}^{\infty} \gamma^k \le D_0^2$ and $\sum_{k=1}^{\infty} \gamma^k k \le D_1^2$. Equations \eqref{eq:eks-final} and \eqref{eq:epsilon-final} are finite with regard to $k$ for a given $s$. Therefore, the error sequences $\mathbb{E}_l\Vert \mathbf{e}^{(k,s)} \Vert$ and $\sqrt{\varepsilon^{(k,s)}}$ in Theorem \ref{thm:transform} satisfy Assumption \ref{ass:summable}. Consequently, Theorem \ref{thm:inexact} also holds, and we in turn conclude that DPSVRG converges to the optimum with the same order as Inexact Prox-SVRG.
		
	\end{proof}
			
			\begin{table}[t]
				\caption{Property of datasets}
				\begin{center}
					\begin{tabular}{|c||c||c||c|}
						\hline
						\textbf{Datasets} & \textbf{\textit{train size}}& \textbf{\textit{feature dimension}}& \textbf{\textit{class}} \\
						\hline
						MNIST & 60,000 & 784 & 10 \\
						\hline
						CIFAR-10 & 50,000 & 1,024 & 10 \\
						\hline
						Adult & 30161 & 30 & 2 \\
						\hline
						Covertype & 100,000 & 54 & 7 \\
						\hline
					\end{tabular}
					\label{table:dataset}
				\end{center}
				\vspace{-0.5cm}
			\end{table}
			
			\begin{figure*}[h]
				\centering
				\subfloat[MNIST]{
					\label{fig:mnist}
					\includegraphics[width=0.245\linewidth]{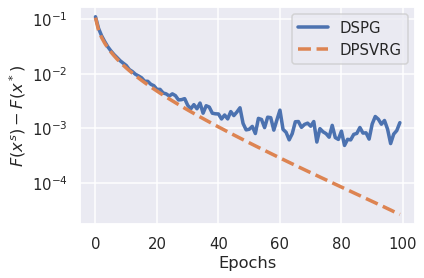}}
				\subfloat[CIFAR-10]{
					\label{fig:cifar10}
					\includegraphics[width=0.245\linewidth]{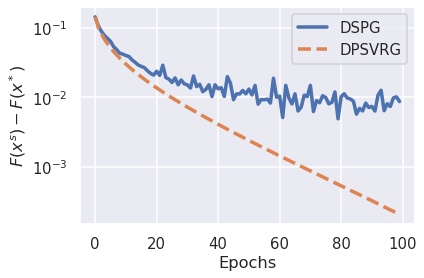}}
				\subfloat[Adult]{
					\label{fig:adult}
					\includegraphics[width=0.245\linewidth]{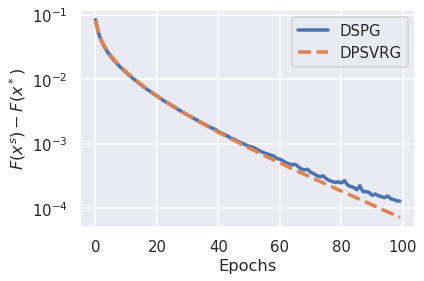}}
				\subfloat[Covertype]{
					\label{fig:covertype}
					\includegraphics[width=0.245\linewidth]{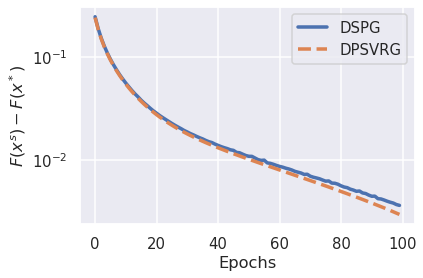}}
				\caption{Training loss of DPSVRG and DSPG}
				\label{fig:loss}
				\vspace{-15pt}
			\end{figure*}
			
			\begin{figure*}[h]
				\centering
				\subfloat[MNIST]{
					\label{fig:mnist-samecomm}
					\includegraphics[width=0.245\linewidth]{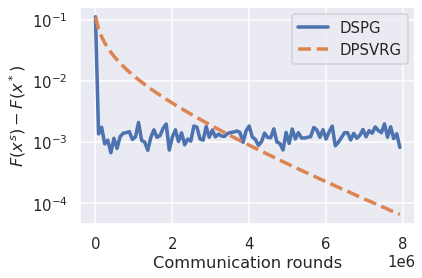}}
				\subfloat[CIFAR-10]{
					\label{fig:cifar10-samecomm}
					\includegraphics[width=0.245\linewidth]{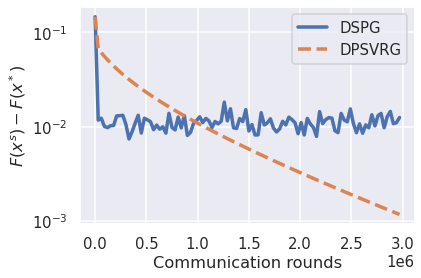}}
				\subfloat[Adult]{
					\label{fig:adult-samecomm}
					\includegraphics[width=0.245\linewidth]{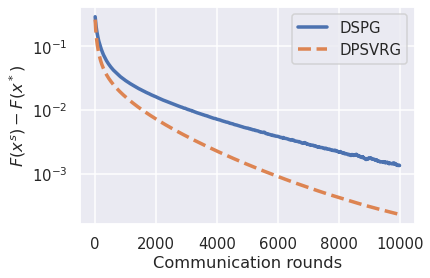}}
				\subfloat[Covertype]{
					\label{fig:covertype-samecomm}
					\includegraphics[width=0.245\linewidth]{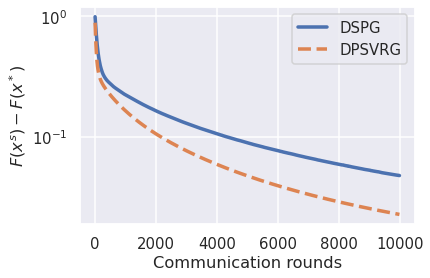}}
				\caption{Training loss of DPSVRG and DSPG with the same communication rounds}
				\label{fig:same-comm}
				\vspace{-15pt}
			\end{figure*}
			
			\begin{figure*}[h]
				\centering
				\subfloat[MNIST]{
					\label{fig:mnist-single-multi}
					\includegraphics[width=0.245\linewidth]{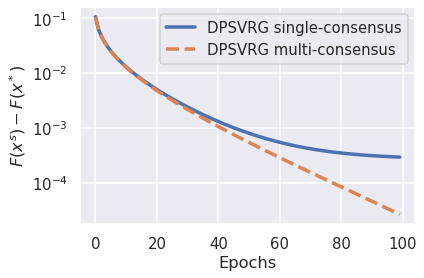}}
				\subfloat[CIFAR-10]{
					\label{fig:cifar10-single-multi}
					\includegraphics[width=0.245\linewidth]{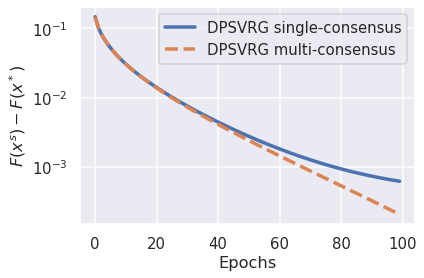}}
				\subfloat[Adult]{
					\label{fig:adult-single-multi}
					\includegraphics[width=0.245\linewidth]{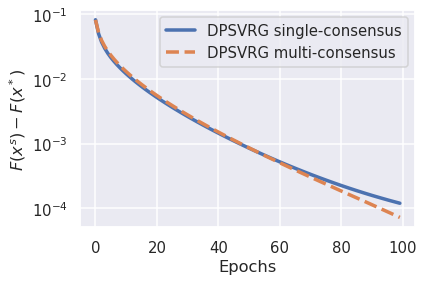}}
				\subfloat[Covertype]{
					\label{fig:covertype-single-multi}
					\includegraphics[width=0.245\linewidth]{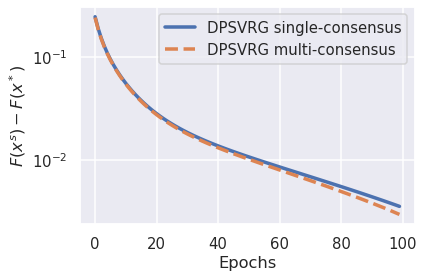}}
				\caption{Training loss of single-consensus and multi-consensus of DPSVRG}
				\label{fig:single-multi}
			\end{figure*}

			\begin{figure*}[h]
				\begin{minipage}[h]{0.5\textwidth}
					\centering
					\includegraphics[width=1\textwidth]{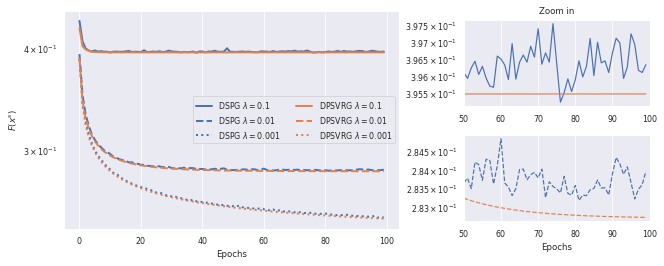}
					\caption{Training loss with different regularizer coefficients ($\lambda$).}
					\label{fig:changelambda}
				\end{minipage}
				\begin{minipage}[h]{0.5\textwidth}
					\centering
					\includegraphics[width=1\textwidth]{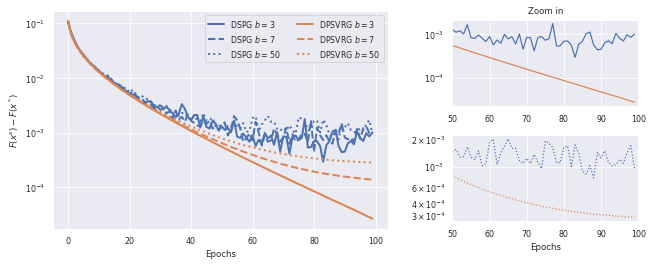}
					\caption{Training loss with different graph connectivity coefficients ($b$).}
					\label{fig:changeB}
				\end{minipage}
			\end{figure*}

		\section{Experimental Results}
		\label{sec:experiments}
		
		\subsection{Evaluation Setup}
		
		\textbf{Testbed Setting and Datasets.} 
		Our experiments are conducted on a testbed of 8 servers, each with an 8-core Intel Xeon W2140b CPU at 3.20GHz, a 32GB DDR4 RAM and a Mellanox Connect-X3 NIC supporting 10 GbE links. We compare the performance of DPSVRG with the baseline DSPG \cite{DSPG} on four typical datasets as listed in Table \ref{table:dataset}.

		We consider a decentralized logistic regression model with an $\ell_1$ norm for binary classification problems with data labels $\{0, 1\}$. The global objective across nodes is expressed as
		\begin{equation}
			\label{eq:objective}
			{\min} \ \frac{1}{mn}\sum_{i=1}^m \sum_{l=1}^{n} \left(-\mathbf{b}_l^i\langle \mathbf{d}_l^i,\mathbf{x} \rangle + {\rm log} (1 + e^{\langle \mathbf{d}_l^i, \mathbf{x}\rangle}) \right) + \lambda \Vert \mathbf{x} \Vert_1.
		\end{equation}
		where $\mathbf{d}_l^i$ and $\mathbf{b}_l^i$ denote the feature and target class of data $\zeta_l^i$ respectively. This objective function satisfies Assumptions \ref{ass:f-property} and \ref{ass:h-property}. For simplicity, all the data is equally partitioned to 8 nodes and each node holds $n$ data samples. If not specified, the graph connectivity coefficient $b$ is set to 1, implying that the network is connected at every step. The impact of $b$ on the convergence performance will also be evaluated in \ref{subsec:graph connectivity}.

		\subsection{Basic Results.}
		We first set $\alpha = 0.01$ and $\lambda = 0.01$ and carry out several experiments of DPSVRG and DSPG on four datasets. The results are displayed in Fig.~\ref{fig:loss}, Fig.~\ref{fig:same-comm} and Fig.~\ref{fig:single-multi}, where the performance is measured by the optimality gap, i.e., $F(\mathbf{x}^s) - F(\mathbf{x}^*)$.  The horizontal coordinate in Fig.~\ref{fig:loss} and Fig.~\ref{fig:single-multi} represents the effective passes through the whole dataset. We execute the centralized gradient method to approximate $F(\mathbf{x}^*)$.

		\textbf{Faster convergence rate.}
		In Fig.~\ref{fig:loss}, the two algorithms converge at the same rate at the beginning, but DSPG oscillates after some epochs and is trapped in a neighborhood of $x^*$ in the end, which is known as the \emph{inexact convergence}. This is due to the large variance of stochastic gradients that lead the parameters to inaccurate stages. One approach to tackle this problem is to adjust the learning rate during training, yet this process is laborious and results in lower convergence rate. On the contrary, DPSVRG converges smoothly while DSPG always oscillates in the MNIST and CIFAR-10 datasets. The convergence rate of DPSVRG is much faster than DSPG in terms of the optimality gap. This advantage is shown in four various datasets. The smoothness of DPSVRG provides us a consistent relationship between the performance and training epochs and can guide us to reasonably control the amount of training rounds. In summary, DPSVRG allows maintaining a constant learning rate to achieve a high and steady convergence speed.
		
		\textbf{Communication-efficiency of multi-consensus.}
		Multi-consensus is used in DPSVRG to speed up the convergence rate by executing communication for multiple rounds in one inner update. One natural question is whether this mechanism raises the communication cost. To illustrate the communication-efficiency of our approach, we compare the optimality gap of two algorithms versus communication rounds in Fig.~\ref{fig:same-comm}. 
		The \emph{inexact convergence} property of DSPG is more evident and cannot be eliminated by increasing the communication rounds.
		It is clear that to reach the optimal point, DPSVRG demands less communication cost than DSPG from a global view. 
		
		We then compare DPSVRG and that embedded with single-consensus to further illustrate the effectiveness of multi-consensus. Fig.~\ref{fig:single-multi} shows that DPSVRG with single-consensus converges a little slower than DPSVRG in terms of the training rounds. Furthermore, by comparing DSPG in Fig.~\ref{fig:loss} and DPSVRG with single-consensus, we observe that DPSVRG converges faster and more smoothly, which demonstrates the effectiveness of variance reduction technique. 
		
		\subsection{Robustness Towards Regularization Coefficient}
		\label{subsec:parameter configuration}
		In order to evaluate the influence of the value of $\ell_1$ on the performance of DPSVRG, we take another experiment by adjusting $\lambda$ from $0.001$ to $0.1$ and fixing $\alpha = 0.01$. A large $\lambda$ will cause an optimal point with sparser elements. The objective functions with varied $\lambda$ are different from each other and the optimal value of these functions are also diverse. Therefore, we use the global loss as the performance metric. From the results in Fig.~\ref{fig:changelambda}, we observe that the value of $\lambda$ has very gentle influence on the stability of DPSVRG. On the contrary, different values of $\lambda$ have a great impact on the behavior of DSPG. In particular, when $\lambda$ is larger ($\lambda=0.1$), the oscillation in DSPG tends to be severer (about $2 \times 10^{-3}$) and the final result resides at a relatively high training loss level (around $3.965 \times 10^{-1}$). In other words, DPSVRG can maintain satisfactory stability with changing $\lambda$ values than its baseline counterpart.
		
		\subsection{Investigating Graph Connectivity.}
		\label{subsec:graph connectivity}
		We study how the performance of our algorithm changes with varied graph connectivity. More specifically, we provide three values of $b$: 3, 7 and 50, where larger value implies sparser graph. With each specific value, a set of $b$ doubly stochastic matrices are pre-determined, satisfying that only the union of all $b$ matrices is connected. Matrices are sampled periodically in the training process, thus the graph $\mathcal{G}$ is $b$-connected. Fig.~\ref{fig:changeB} shows the optimality gap of DSPG and DPSVRG versus the outer iteration rounds on MNIST dataset. On the sparser communication graph, both two algorithms converge slower and their performance gap becomes wider. This result can be clearly observed if we zoom in on the optimality gap over the last $50$ iterations. The performance difference between two approaches increases from $1 \times 10^{-3}$ to $2 \times 10^{-3}$ with $b$ changing from $3$ to $50$. In addition, DSPG over sparser graph has more violent oscillations and is trapped in the neighborhood farther from the optimal point (around $1.5 \times 10^{-3}$ when $b=50$ compared to $0.9 \time 10^{-3}$ when $b=3$). Meanwhile, the sparsity of communication graphs only influences the convergence speed of DPSVRG, but does not prevent it converging to the optimum. This robustness to different sparsity degrees of the graph improves the practicability of DPSVRG in sparser and time-varying networks.

		\section{Conclusion}
		\label{sec:conclusion}
		In this paper, we propose DPSVRG to efficiently speedup the convergence of the decentralized stochastic proximal descent algorithm. DPSVRG leverages the variance reduction technique to reduce the variance of the gradient estimator 
		and rectify the imprecise update direction. We transform DPSVRG into its centralized equivalent that is the inexact Prox-SVRG algorithm under mild conditions. We prove that DPSVRG and Inexact Prox-SVRG have an $O(\frac{1}{T})$ convergence rate for a general smooth and convex objective function plus a non-smooth regularization term, while DSPG merely converges at an $O(\frac{1}{\sqrt{T}})$ rate. 
		Experimental results demonstrate that DPSVRG converges much faster than DSPG with different time-varying topologies and different regularization coefficients. Especially, the training loss of DPSVRG decreases smoothly as the training epoch moves on, while that of DSPG exhibits obvious fluctuations.

		\appendix
		
		\section{Proof of Theorem \ref{thm:transform}}
		In the initial state, we have $\bar{\mathbf{x}}^{(0,1)} = \mathbf{x}^{init} = \mathbf{x}^{(0,1)}$ and $\tilde{\mathbf{x}}^{0} = \frac{1}{m} \sum_{i=1}^{m} \tilde{\mathbf{x}}_i^{0}$. Hence, if $\mathbf{x}^{(k,1)} = \bar{\mathbf{x}}^{(k,1)}, \forall k \in [1,K_1]$ can be satisfied, we have
		\vspace{-5pt}
		\begin{equation*}
			\begin{aligned}
				\tilde{\mathbf{x}}^1 = \frac{1}{K_1} \sum_{k=1}^{K_1} \mathbf{x}^{(k,1)} = \frac{1}{K_1} \sum_{k=1}^{K_1} \bar{\mathbf{x}}^{(k,1)} = \frac{1}{m} \sum_{i=1}^{m} \tilde{\mathbf{x}}_i^{1}, \\
				\mathbf{x}^{(0,2)} = \mathbf{x}^{(K_1,1)} = \bar{\mathbf{x}}^{(K_1,1)} = \frac{1}{m} \sum_{i=1}^{m} \mathbf{x}^{(K_1,1)}= \bar{\mathbf{x}}^{(0,2)}.
			\end{aligned}
			\vspace{-5pt}
		\end{equation*}
		
		In this way, $\mathbf{x}$ and $\tilde{\mathbf{x}}$ successfully track $\bar{\mathbf{x}}$ and $\frac{1}{m} \sum_{i=1}^{m} \tilde{\mathbf{x}}_1$ in the first inner loop. By repeating the update rules in Algorithm \ref{alg:inexact prox-svrg} for $S$ rounds, the two algorithms will reach the same optimal point $\tilde{\mathbf{x}}^S$. Clearly, given $s$, an explicit method to guarantee $\mathbf{x}^{(k,s)} = \bar{\mathbf{x}}^{(k,s)}, \forall k$ is to ensure $\mathbf{q}^{(k,s)} = \bar{\mathbf{q}}^{(k,s)}$. Therefore, in the following we use equations $\mathbf{x}^{(k,s)} = \bar{\mathbf{x}}^{(k,s)}$ and $\mathbf{q}^{(k,s)} = \bar{\mathbf{q}}^{(k,s)}$ to deduce the required expression of error terms $\mathbf{e}^{(k,s)}$ and $\varepsilon^{(k,s)}$. 
		
		Consider the updates of $\mathbf{q}_i^{(k,s)}$ (lines $7\sim 8$, Algorithm \ref{alg:prox-svrg}) and $\mathbf{q}^{(k,s)}$ (lines $7\sim 8$, Algorithm \ref{alg:inexact prox-svrg}), respectively.
		Taking the average of $\mathbf{q}_i^{(k,s)}$ over $i$ yields the expression of $\bar{\mathbf{q}}^{(k,s)}$ as below.
		\vspace{-10pt}
		\begin{equation*}
			\begin{aligned}
				\bar{\mathbf{q}}^{(k,s)} = &\bar{\mathbf{x}}^{(k-1,s)} - \alpha \Big(\frac{1}{m}\sum_{i=1}^m \nabla f_i^{l_i}(\mathbf{x}_i^{(k-1, s)})\\
				&- \frac{1}{m}\sum_{i=1}^m \nabla f_i^{l_i}(\tilde{\mathbf{x}}_i^{s-1}) 
				+ \frac{1}{m}\sum_{i=1}^m \nabla f_i(\tilde{\mathbf{x}}_i^{s-1})\Big).
			\end{aligned}
		\end{equation*}
		
		The error term $\mathbf{e}^{(k,s)}$ in \eqref{eq:e} is generated from the difference of $\bar{\mathbf{q}}^{(k,s)}$ and $\mathbf{q}^{(k,s)}$.
		
		Next we derive $\varepsilon^{(k,s)}$ by starting with the inexact proximal operator in line 9 of Algorithm \ref{alg:inexact prox-svrg}. According to \eqref{eq:inexact prox}, we have
		\begin{equation}
			\label{eq:inexact epsilon}
			\begin{aligned}
				&\frac{1}{2\alpha}\Vert \mathbf{x}^{(k,s)}-\bar{\mathbf{q}}^{(k,s)}\Vert^2 + h(\mathbf{x}^{(k,s)}) \\
				&\le \frac{1}{2\alpha}\Vert \mathbf{y}^{(k,s)}-\bar{\mathbf{q}}^{(k,s)}\Vert^2 + h(\mathbf{y}^{(k,s)}) + \varepsilon^{(k,s)},
			\end{aligned}
		\end{equation}
		with $\mathbf{y}^{(k,s)} = {\rm prox}_h^{\alpha} \{\bar{\mathbf{q}}^{(k,s)}\}$, which implies (with Lemma \ref{lemma:prox-property})
		\begin{equation}
			\label{eq:1/alpha}
			\frac{1}{\alpha}(\bar{\mathbf{q}}^{(k,s)}-\mathbf{y}^{(k,s)}) \in \partial h(\mathbf{y}^{(k,s)}).
		\end{equation}
		
		We use the updates in line $9 \sim 10$ of Algorithm \ref{alg:prox-svrg} to construct the inequality similar to \eqref{eq:inexact epsilon}.
		\begin{align}
			\label{eq:1/2alpha}
			&\frac{1}{2\alpha}\Vert\bar{\mathbf{x}}^{(k,s)}-\bar{\mathbf{q}}^{(k,s)}\Vert^2 + h(\bar{\mathbf{x}}^{(k,s)}) \notag\\
			&= \frac{1}{2\alpha}\Vert\bar{\mathbf{x}}^{(k,s)}-\mathbf{y}^{(k,s)}\Vert^2 + \frac{1}{2\alpha}\Vert \mathbf{y}^{(k,s)}-\bar{\mathbf{q}}^{(k,s)}\Vert^2  + h(\bar{\mathbf{x}}^{(k,s)}) \notag\\
			&\quad + \frac{1}{\alpha}\langle \bar{\mathbf{x}}^{(k,s)}-\mathbf{y}^{(k,s)}, \mathbf{y}^{(k,s)}-\bar{\mathbf{q}}^{(k,s)} \rangle \notag\\
			&\le \frac{1}{2\alpha}\Vert \mathbf{y}^{(k,s)}-\bar{\mathbf{q}}^{(k,s)}\Vert^2 + h(\mathbf{y}^{(k,s)}) + \frac{1}{2\alpha}\Vert\bar{\mathbf{x}}^{(k,s)}-\mathbf{y}^{(k,s)}\Vert^2 \notag\\
			&\quad + \frac{1}{\alpha}\langle \bar{\mathbf{x}}^{(k,s)}-\mathbf{y}^{(k,s)}, \mathbf{y}^{(k,s)}-\bar{\mathbf{q}}^{(k,s)} \rangle 
			- \langle \mathbf{p}, \mathbf{y}^{(k,s)}-\bar{\mathbf{x}}^{(k,s)} \rangle,
		\end{align}
		
		where $p$ is a vector with $\mathbf{p} \in \partial h(\bar{\mathbf{x}}^{(k,s)})$ and the inequality holds due to the convexity of $h$. Comparing \eqref{eq:1/2alpha} with \eqref{eq:inexact epsilon}, we will get the form of $\varepsilon^{(k,s)}$ in \eqref{eq:varepsilon}, thus completing the proof of Theorem \ref{thm:transform}.

		\section{Proofs of Lemmas}
		
		\subsection{Proof of Lemma \ref{lemma:g}}
		\label{appendix:lemma-g}
		Given the definition of $\partial_{\varepsilon}h(z)$ in \eqref{eq:subdifferential}, we note that if convex function $h$ can be divided into two convex parts $h=h_1+h_2$, then $\partial_{\varepsilon}h(z) \subset \partial_{\varepsilon}h_1(z) + \partial_{\varepsilon}h_2(z)$. 
		
		Applying \eqref{eq:subdifferential} to $x = {\rm prox}_{h,\varepsilon}^{\alpha} \{z\}$, we obtain $0\in \partial_{\varepsilon}\{\frac{1}{2\alpha}\Vert x-z\Vert^2+h(x)\}$. Therefore, there exists $t$ and $-t$, so that $t\in \partial_{\varepsilon}\{\frac{1}{2\alpha}\Vert x-z\Vert^2\}$ and $-t\in \partial_{\varepsilon}h(x)$. According to the definition of $\varepsilon$-subdifferential, we have that $\forall{y}$,
		\begin{equation*}
			\begin{aligned}
				\frac{1}{2\alpha} \Vert y-z \Vert^2 \ge \frac{1}{2\alpha} \Vert x-z\Vert^2 + \langle t, y-x \rangle - \varepsilon.
			\end{aligned}
		\end{equation*}
		
		Unfolding $\frac{1}{2\alpha}\Vert y-z \Vert^2$ and applying $-\langle y,z \rangle \le \Vert y \Vert \cdot \Vert z \Vert$ and $\langle t, y \rangle \ge -\Vert y \Vert \cdot \Vert t \Vert$ yields a quadratic inequality in terms of $y$. To make sure this inequality holds for $\forall y$, we let the discriminant nonpositive ($\Delta \le 0$), i.e., 
		\begin{equation*}
			\begin{aligned}
				\left( \frac{1}{\alpha}\Vert z \Vert+\Vert t \Vert \right)^2 \le -\frac{2}{\alpha}\left( \frac{1}{2\alpha}\Vert x-z\Vert^2 -\langle t, x \rangle -\frac{1}{2\alpha}\Vert z \Vert^2 - \varepsilon \right).
			\end{aligned}
		\end{equation*}
		Simplifying this equation with $\frac{2}{\alpha}\Vert z \Vert\cdot\Vert t \Vert \ge -\frac{2}{\alpha}\langle z, t \rangle$ produces
		\begin{equation*}
			\frac{1}{2\alpha} \Vert x-(z+\alpha t) \Vert^2 \le \varepsilon.
		\end{equation*}
		Defining $g = x - (z+\alpha t)$,  Lemma \ref{lemma:g} readily follows.

		\subsection{Proof of Lemma \ref{lemma:inexact nonexpansive}}
		\label{appendix:lemma 3}
		Let $x_1 = prox_{h,\varepsilon_1}^{\alpha} \{z_1\}$, $x_2 = prox_{h,\varepsilon_2}^{\alpha} \{z_2\}$. According to Lemma \ref{lemma:g}, we have
		\begin{equation*}
			\begin{aligned}
				&p_1 = \frac{z_1+g_1-x_1}{\alpha}, \ p_1 \in \partial_{\varepsilon_1}h(z_1), \ \Vert g_1 \Vert \le \sqrt{2\alpha \varepsilon_1}, \\
				&p_2 = \frac{z_2+g_2-x_2}{\alpha}, \ p_2 \in \partial_{\varepsilon_2}h(z_2), \ \Vert g_2 \Vert \le \sqrt{2\alpha \varepsilon_2}.
			\end{aligned}
		\end{equation*}
		With the property of $\varepsilon$-subdifferential, we obtain that $\forall{y}$,
		\begin{subequations}
			\begin{equation}
				\quad h(y) \ge h(x_1) + \langle p_1, y-x_1 \rangle - \varepsilon_1, \label{eq:1}
			\end{equation}
			\begin{equation}
				\quad h(y) \ge h(x_2) + \langle p_2, y-x_2 \rangle - \varepsilon_2. \label{eq:2}
			\end{equation}
		\end{subequations}
		We replace $y$ with $x_2$ in \eqref{eq:1} and with $x_1$ in \eqref{eq:2} and get
		\begin{gather*}
			h(x_1) \ge h(x_2) + \frac{1}{\alpha}\langle z_2+g_2-x_2, x_1-x_2 \rangle - \varepsilon_2, \\
			h(x_2) \ge h(x_1) + \frac{1}{\alpha}\langle z_1+g_1-x_1, x_2-x_1 \rangle - \varepsilon_1.
		\end{gather*}
		Adding up these two inequalities leads to
		\begin{equation*}
			\langle z_2-z_1+g_2-g_1, x_1-x_2 \rangle + \Vert x_1-x_2 \Vert^2 - \alpha (\varepsilon_1+\varepsilon_2) \le 0.
		\end{equation*}
		Therefore,
		\begin{equation*}
			\begin{aligned}
				&\Vert x_1-x_2 \Vert \\
				\le &\frac{\Vert g_1-g_2+z_1-z_2\Vert +\sqrt{\Vert g_1-g_2+z_1-z_2\Vert^2 + 4\alpha (\varepsilon_1+\varepsilon_2)}}{2} \\
				\le &\Vert z_1-z_2\Vert + \sqrt{2\alpha \varepsilon_1} + \sqrt{2\alpha \varepsilon_2} + \sqrt{\alpha (\varepsilon_1+\varepsilon_2)}
			\end{aligned}
		\end{equation*}
		We complete the proof by taking $\varepsilon_1 = \varepsilon_2 = \varepsilon$.

		\subsection{Proof of Lemma \ref{lemma:svrg}}
		\label{appendix:lemmam 4}
		The proof of Lemma \ref{lemma:svrg} is analogous to most of the SVRG literature with some needed differences.
		\begin{equation}
			\label{eq:svrg-1}
			\begin{aligned}
				&\mathbb{E}_{l_{in}} \Vert v^{(k)}-\nabla f(x^{(k-1)}) \Vert^2 \\
				&= \mathbb{E}_{l_{in}}\Big\Vert \nabla f^{l_{in}}(x^{(k-1)}) - \nabla f^{l_{in}}(\tilde{x}^{s-1})  \\
				& \quad + \nabla f(\tilde{x}^{s-1})- \nabla f(x^{(k-1)})\Big\Vert^2 \\
				&\le \mathbb{E}_{l_{in}} \Vert \nabla f^{l_{in}}(x^{(k-1)}) - \nabla f^{l_{in}}(\tilde{x}^{s-1}) \Vert^2 \\
				& \le2\mathbb{E}_{l_{in}}\Vert\nabla f^{l_{in}}(x^{(k-1)}) - \nabla f^{l_{in}}(x^*)\Vert^2 \\
				&\quad + 2\mathbb{E}_{l_{in}}\Vert\nabla f^{l_{in}}(\tilde{x}^{s-1}) - \nabla f^{l_{in}}(x^*)\Vert^2\\
				& \le \frac{2}{m} \sum_{l=1}^{m} \mathbb{E}_l \Vert\nabla f^l(x^{(k-1)}) - \nabla f^l(x^*)\Vert^2 \\
				&\quad +\frac{2}{m} \sum_{l=1}^{m} \mathbb{E}_l \Vert\nabla f^l(\tilde{x}^{s-1}) - \nabla f^l(x^*)\Vert^2
			\end{aligned}
		\end{equation}
		The first inequality follows from ${\rm Var}[x] \le \mathbb{E}[x^2]$ and notation $l$ in the last formula refers to the sample $l$. Now let $\phi^l(x) = f^l(x) - f^l(x^*) - \langle f^l(x^*), x-x^* \rangle$. Clearly, the convexity of $f^l(x)$ leads to $\phi^l(x)\ge0$. Then we acquire
		\begin{equation}
			\label{eq:svrg-2}
			\begin{aligned}
				0 &= min_x f^l(x) \le min_{x, \mu}f^l\big(x-\mu \nabla f^l(x)\big) \\
				&\le min_{x,\mu} f^l(x) + \langle \nabla f^l(x), -\mu \nabla f^l(x) \rangle + \frac{L}{2} \Vert-\mu \nabla f^l(x)\Vert^2 \\
				&= min_{x, \mu}f^l(x) - \mu \Vert\nabla f^l(x)\Vert^2 + \frac{L}{2}\mu^2\Vert\nabla f^l(x)\Vert^2\\
				&= min_x f^l(x) - \frac{1}{2L}\Vert\nabla f^l(x)\Vert^2,
			\end{aligned}
		\end{equation}
		which means $\Vert\nabla f^l(x)\Vert^2 \le 2L f^l(x)$, or in other words, $\Vert\nabla f^l(x) - \nabla f^l(x^*)\Vert^2 \le 2L(f^l(x)-f^l(x^*)-\langle \nabla f^l(x^*), x-x^* \rangle)$. Assume that $n$ data samples are i.i.d, then we have
		\begin{equation}
			\label{eq:svrg-3}
			\begin{aligned}
				&\mathbb{E}_l\Vert\nabla f^l(x) - \nabla f^l(x^*)\Vert^2 = \frac{1}{n}\sum_{j=1}^n\Vert\nabla f^j(x) - \nabla f^j(x^*)\Vert^2 \\
				\le &2L(\frac{1}{n}\sum_{j=1}^n f^j(x)-\frac{1}{n}\sum_{j=1}^n f^j(x^*)-\frac{1}{n}\sum_{j=1}^n \langle \nabla f^j(x^*), x-x^* \rangle) \\
				\le &2L(f(x)-f(x^*)-\langle \nabla f(x^*), x-x^* \rangle)\\
				= &2L(f(x)-f(x^*)+\langle u, x-x^* \rangle) \\
				\le &2L(f(x)-f(x^*)+h(x)-h(x^*)) \\
				= &2L(F(x)-F(x^*)),
			\end{aligned}
		\end{equation}
		where $f(x) = \frac{1}{n}\sum_{j=1}^n f^j(x)$, $F(x) = f(x) + h(x)$ and $u \in \partial h(x^*)$, which satisfies $\nabla f(x^*) + u = 0$, where $\mathbf{x}^*$ denotes the optimal point of \eqref{P2}. The last inequality is due to the convexity of $h(x)$. We obtain Lemma \ref{lemma:svrg} by substituting \eqref{eq:svrg-3} into \eqref{eq:svrg-1}.

		\subsection{Proof of Lemma \ref{lemma:uk}}
		\label{appendix:lemma-uk}
		It is obvious that the result is true for $K=0$. By induction technique, if \eqref{eq:uK} is true for $K-1$, then we need to prove it is also true for $K$. From the recursion, we successively attain the following inequalities.
		\begin{subequations}
			\begin{equation}
				u_{K}^2 \le S_{K} + \sum_{k=1}^{K}\lambda_k u_k = S_{K} + \sum_{k=1}^{K-1}\lambda_k u_k + \lambda_{K}u_{K}
			\end{equation}
			\begin{equation}
				u_{K}^2-\lambda_{K}u_{K} \le S_{K} + \sum_{k=1}^{K-1}\lambda_k u_k
			\end{equation}
			\begin{equation}
				(u_{K} - \frac{1}{2}\lambda_K)^2 \le S_{K} + \sum_{k=1}^{K-1}\lambda_k u_k + \frac{1}{4}\lambda_K^2
			\end{equation}
			\begin{equation}
				u_{K} \le \frac{1}{2}\lambda_K + \left(S_{K} + \sum_{k=1}^{K-1}\lambda_k u_k + \frac{1}{4}\lambda_K^2\right)^{\frac{1}{2}} \label{eq:4equations-4}
			\end{equation}
		\end{subequations}
		
		Let $d_{K-1} = max\{u_0, u_1, ..., u_{K-1}\}$. With $u_{K-1} \le d_{K-1}$, \eqref{eq:4equations-4} implies
		\begin{equation*}
			\begin{aligned}
				u_{K} \le \frac{1}{2}\lambda_K + \left(S_{K} + d_{K-1}\sum_{k=1}^{K-1}\lambda_k + \frac{1}{4}\lambda_K^2\right)^{\frac{1}{2}},
			\end{aligned}
		\end{equation*}
		which in turn leads to
		\begin{equation}
			\label{eq:dK-max}
			\begin{aligned}
				d_K &= {\rm max}\{d_{K-1}, u_{K}\} \\
				&\le {\rm max}\left\{d_{K-1}, \underbrace{\frac{1}{2}\lambda_K + \left(S_{K} + d_{K-1}\sum_{k=1}^{K-1}\lambda_k + \frac{1}{4}\lambda_K^2\right)^{\frac{1}{2}}}\right\}.
			\end{aligned}
		\end{equation}
		
		Define the second term in ${\rm max}$ operator as $\Delta$. We then analyze \eqref{eq:dK-max} in three situations.
		(a) If $d_{K-1} = \Delta$, we get
		$d_{K-1}^* = \frac{1}{2}\sum_{k=1}^{K}\lambda_k + \left(S_K + (\frac{1}{2}\sum_{k=1}^{K}\lambda_k)^2\right)^{\frac{1}{2}}$ and $
		u_K \le d_K \le d_{K-1}^*$. (b) If $d_{K-1} < \Delta$, then $d_{K-1} < d_{K-1}^*$, leading to $u_K \le d_K \le \Delta$. Taking $d_{K-1}$ as the variable, we obtain $\Delta|_{d_{K-1} < d_{K-1}^*} < \Delta|_{d_{K-1} = d_{K-1}^*} = d_{K-1}^*$. Thus $u_K < d_{K-1}^*$. (c) If $d_{K-1} > \Delta$, by induction method, we utilize the conclusions from $k=0$ to $K-1$ and acquire
		\begin{equation*}
			\begin{aligned}
				u_K &\le d_K \le d_{K-1} \\
				&\le max \left\{S_0^{\frac{1}{2}},..., \quad \frac{1}{2}\sum_{k=1}^{K-1}\lambda_k + \left(S_{K-1} + (\frac{1}{2}\sum_{k=1}^{K-1}\lambda_k)^2\right)^{\frac{1}{2}}\right\} \\
				&\le \frac{1}{2}\sum_{k=1}^{K}\lambda_k + \left(S_K + (\frac{1}{2}\sum_{k=1}^{K}\lambda_k)^2\right)^{\frac{1}{2}}.
			\end{aligned}
		\end{equation*}

		\section{Proofs of Propositions}
		
		\subsection{Proof of Proposition \ref{prop:polynomial}}
		The update rules of Algorithm \ref{alg:prox-svrg} (line $8 \sim 10$) imply that
		\begin{equation}
			\label{eq:abc}
			\begin{aligned}
				&(a) \ \Vert q_i^{(k,s)}\Vert = \Vert x_i^{(k-1,s)}-\alpha v_i^{(k,s)}\Vert 
				\overset{Ass \ref{ass:gradient bound}}{\le} \Vert x_i^{(k-1,s)}\Vert+\alpha G_g, \\
				&(b) \ \Vert x_i^{(k-1,s)}\Vert \le \Vert\hat{q}_i^{(k-1,s)}\Vert + \alpha G_h, \\
				&(c) \ \sum_{i=1}^{m} \Vert\hat{q_i}^{(k-1,s)}\Vert \le \sum_{i=1}^{m}\sum_{j=1}^{m}\phi_{i,j}^{(k,s)} \Vert q_j^{(k-1,s)}\Vert = \sum_{j=1}^{m} \Vert q_j^{(k-1,s)}\Vert,
			\end{aligned}
		\end{equation}
		leading to
		\begin{equation}
			\label{eq:sumqiks}
			\sum_{i=1}^{m} \Vert q_i^{(k,s)}\Vert \le \sum_{i=1}^{m} \Vert q_i^{(k-1,s)}\Vert + \alpha m (G_g + G_h).
		\end{equation}
		The inequality in \eqref{eq:abc}$(b)$ generated from the same technique with \eqref{eq:1/alpha}. More specifically, owing to $x_i^{(k,s)} = prox_h^{\alpha}\{\hat{q_i}^{(k,s)}\}$, we have $\frac{1}{\alpha}(\hat{q_i}^{(k-1,s)}-x_i^{(k,s)}) \in \partial h(x_i^{(k,s)})$, leading to $\Vert \hat{q_i}^{(k,s)} - x_i^{(k,s)} \Vert \le G_h$ and then $\Vert x_i^{(k,s)} \Vert \le \Vert \hat{q_i}^{(k,s)} \Vert + G_h$ 
		
		It is obvious that Proposition \ref{prop:polynomial} holds with $k=s=1$. By induction method, we are required to demonstrate that if \eqref{eq:polynomial} holds for every positive integer tuple $(k,s)$, it will also be satisfied for $(k+1,s)$ and $(k,s+1)$. The former one can be proved by directly utilizing \eqref{eq:sumqiks}, i.e.,
		\begin{equation}
			\begin{aligned}
				\sum_{i=1}^m\Vert q_i^{(k+1,s)}\Vert
				&\le \sum_{i=1}^m\Vert q_i^{(k,s)}\Vert+\alpha m(G_g+G_h)\\
				&\le C_0 + C_1k +C_2s + \alpha m(G_g+G_h) \\
				&= C_0 + C_1(k+1) + C_ss,
			\end{aligned}
		\end{equation}
		and the latter one is dealt with similar method
		\begin{align}
			\sum_{i=1}^m\Vert q_i^{(k,s+1)}\Vert
			&\le \sum_{i=1}^m\Vert q_i^{(1,s+1)}\Vert+\alpha m(k-1)(G_g+G_h) \notag\\
			&= \sum_{i=1}^m\Vert q_i^{(K_s,s)}\Vert+\alpha m(k-1)(G_g+G_h) \notag\\
			&\le \sum_{i=1}^m\Vert q_i^{(k,s)}\Vert+\alpha m(K_s-1)(G_g+G_h) \notag\\
			&\le C_0+C_1k+C_2s+\alpha m(K_s-1)(G_g+G_h) \notag\\
			&\le C_0+C_1k+C_2s+\alpha m \beta_s n_0 (G_g+G_h) \notag\\
			&= C_0 + C_1k + C_2(s+1).
		\end{align}
		Therefore, the hypothesis induction holds for both $(k+1,s)$ and $(k,s+1)$ and Proposition \ref{prop:polynomial} is proven.

		\subsection{Proof of Proposition \ref{prop:xik-1}}
		We first unfold $\xi^{(k-1)}$ in $\langle \xi^{(k-1)}, x^{(k)}-x^* \rangle$ and obtain
		\begin{align}
			\label{eq:xik-1,xk-x*}
			&\langle \xi^{(k-1)}, x^{(k)}-x^* \rangle \notag\\
			&= \langle x^{(k)}-x^*, w + v^{(k)} + e^{(k)} - \frac{1}{\alpha}g \rangle \notag\\
			&= \langle x^{(k)}-x^*, w \rangle + \langle x^{(k)}-x^*, v^{(k)} \rangle + \langle x^{(k)}-x^*, e^{(k)} - \frac{1}{\alpha}g \rangle \notag\\
			&\overset{\eqref{eq:subdifferential}}{\ge} h(x^{(k)}) - h(x^*) - \varepsilon^{(k)} + \langle x^{(k)}-x^*, v^{(k)} \rangle \notag\\
			&\quad + \langle x^{(k)}-x^*, e^{(k)} - \frac{1}{\alpha}g \rangle \notag\\
			&= h(x^{(k)}) - h(x^*) - \varepsilon^{(k)} +\langle x^{(k)}-x^*, v^{(k)} - \nabla f(x^{(k-1)}) \rangle \notag\\
			&\quad + \underbrace{\langle x^{(k)}-x^*, \nabla f(x^{(k-1)}) \rangle} + \langle x^{(k)}-x^*, e^{(k)} - \frac{1}{\alpha}g \rangle.
		\end{align}
		Next we deal with the last second term in the following way.
		\begin{align}
			\label{eq:xk-x*,deltaf}
			&\langle x^{(k)}-x^*, \nabla f(x^{(k-1)}) \rangle \notag\\
			&= \langle x^{(k)}-x^{(k-1)}, \nabla f(x^{(k-1)}) \rangle + \langle x^{(k-1)}-x^*, \nabla f(x^{(k-1)}) \rangle \notag\\
			&\overset{\eqref{eq:f convex}\eqref{eq:smooth}}{\ge} f(x^{(k)}) - f(x^{(k-1)}) - \frac{L}{2}\Vert x^{(k)}-x^{(k-1)}\Vert^2 + f(x^{(k-1)}) - f(x^*) \notag\\
			&\overset{\eqref{eq:xi}}{=} f(x^{(k)}) - f(x^*) - \frac{L}{2}\alpha^2\Vert\xi^{(k-1)}\Vert^2
		\end{align}
		Substituting \eqref{eq:xk-x*,deltaf} into \eqref{eq:xik-1,xk-x*}, we acquire
		\begin{align}
			&\langle \xi^{(k-1)}, x^{(k)}-x^* \rangle \notag \\
			&\ge h(x^{(k)}) - h(x^*) - \varepsilon^{(k)} +\langle x^{(k)}-x^*, v^{(k)} - \nabla f(x^{(k-1)}) \rangle \notag\\
			&\quad + f(x^{(k)}) - f(x^*) - \frac{L}{2}\alpha^2\Vert\xi^{(k-1)}\Vert^2 + \langle x^{(k)}-x^*, e^{(k)} - \frac{1}{\alpha}g \rangle \notag\\
			&= F(x^{(k)}) - F(x^*)- \varepsilon^{(k)} +\langle x^{(k)}-x^*, v^{(k)} - \nabla f(x^{(k-1)}) \rangle \notag\\
			&\quad - \frac{L}{2}\alpha^2\Vert\xi^{(k-1)}\Vert^2 + \langle x^{(k)}-x^*, e^{(k)} - \frac{1}{\alpha}g \rangle,
		\end{align}
		leading to
		\begin{align}
			\label{eq:alpha2}
			&-\alpha^2 \Vert\xi^{(k-1)}\Vert^2 - 2\alpha \langle \xi^{(k-1)}, x^{(k)}-x^* \rangle \notag\\
			&\le -\alpha^2 \Vert\xi^{(k-1)}\Vert^2 - 2\alpha (F(x^{(k)}) - F(x^*)) + 2\alpha \varepsilon^{(k)} \notag\\ &\quad - 2\alpha \langle x^{(k)}-x^*, v^{(k)} - \nabla f(x^{(k-1)}) \rangle \notag\\
			& \quad + L\alpha^3\Vert\xi^{(k-1)}\Vert^2 - 2\alpha \langle x^{(k)}-x^*, e^{(k)} - \frac{1}{\alpha}g \rangle \notag\\
			&\le - 2\alpha (F(x^{(k)}) - F(x^*)) + 2\alpha \varepsilon^{(k)} \notag\\ 
			&\quad \underbrace{- 2\alpha \langle x^{(k)}-x^*, v^{(k)} - \nabla f(x^{(k-1)}) \rangle} \notag\\
			&\quad - 2\alpha \langle x^{(k)}-x^*, e^{(k)} - \frac{1}{\alpha}g \rangle,
		\end{align}
		where the second inequality is due to $L\alpha \le 1$. Next we define a new variable $\underline{x} = prox_{h, \varepsilon^{(k)}}^{\alpha}\{x^{(k-1)}-\alpha \nabla f(x^{(k-1)})\}$. Hence the third term on the right of \eqref{eq:alpha2} can be divided into $- 2\alpha \langle x^{(k)}-\underline{x}, v^{(k)} - \nabla f(x^{(k-1)}) \rangle$ and $- 2\alpha \langle \underline{x}-x^*, v^{(k)} - \nabla f(x^{(k-1)}) \rangle$. The former satisfies
		\begin{equation}
			\label{eq:xk-x,vk-deltaf}
			\begin{aligned}
				&- 2\alpha \langle x^{(k)}-\underline{x}, v^{(k)} - \nabla f(x^{(k-1)}) \rangle \\
				&\le 2\alpha \Vert x^{(k)}-\bar{x}\Vert \cdot \Vert v^{(k)} - \nabla f(x^{(k-1)})\Vert \\
				&\le 2\alpha \left(\alpha\Vert v^{(k)} - \nabla f(x^{(k-1)})\Vert +\alpha\Vert e^{(k)}\Vert +3 \sqrt{2\alpha\varepsilon^{(k)}}\right) \cdot \\
				&\quad \Vert v^{(k)} - \nabla f(x^{(k-1)})\Vert \\
				&= 2\alpha^2\Vert v^{(k)} - \nabla f(x^{(k-1)})\Vert^2 + (2\alpha^2\Vert e^{(k)}\Vert+6\alpha \sqrt{2\alpha\varepsilon^{(k)}}) \cdot \\
				&\quad \Vert v^{(k)} - \nabla f(x^{(k-1)})\Vert \\
				&\le 2\alpha^2\Vert v^{(k)} - \nabla f(x^{(k-1)})\Vert^2 + (2\alpha^2\Vert e^{(k)}\Vert+6\alpha \sqrt{2\alpha\varepsilon^{(k)}})M.
			\end{aligned}
		\end{equation}
		The second and third inequalities follow from Lemma \ref{lemma:inexact nonexpansive} and Assumption \ref{ass:M-bound} respectively. We complete the proof by plugging \eqref{eq:xk-x,vk-deltaf} into \eqref{eq:alpha2}.
		
		\subsection{Proof of Proposition 3}
		We apply Lemma \ref{lemma:uk} with $\lambda_k = 2\alpha\mathbb{E}\Vert \mathbf{e}^{(k)}\Vert+2\sqrt{2\alpha\varepsilon^{(k)}}$ and $u_k = \mathbb{E}\Vert \mathbf{x}^{(k)}-\mathbf{x}^*\Vert$ and have
		\begin{align}
			\label{eq:ExK-x*}
			&\mathbb{E} \Vert \mathbf{x}^{(K)}-\mathbf{x}^*\Vert \le \sum_{k=1}^K\big(\alpha\Vert \mathbf{e}^{(k)}\Vert+\sqrt{2\alpha\varepsilon^{(k)}}\big)+ \Bigg(\Vert \mathbf{x}^{(0)}-\mathbf{x}^*\Vert^2 \notag\\
			&\quad + 8\alpha^2L\big(F(\mathbf{x}^{(0)})-F(\mathbf{x}^*)\big)
			+ 8\alpha^2LK\big(F(\tilde{\mathbf{x}})-F(\mathbf{x}^*)\big) \notag\\
			&\quad+ 2\alpha^2M\sum_{k=1}^K\mathbb{E}\Vert \mathbf{e}^{(k)}\Vert + 6\alpha M \sqrt{2\alpha}\sum_{k=1}^K\sqrt{\varepsilon^{(k)}} + 2\alpha\sum_{k=1}^K\varepsilon^{(k)} \notag\\
			&\quad+ \left(\sum_{k=1}^K\big(\alpha\Vert \mathbf{e}^{(k)}\Vert+\sqrt{2\alpha\varepsilon^{(k)}}\big)\right)^2 \Bigg)^{1/2} \notag\\
			&\le A_K + \Big(\Vert \mathbf{x}^{(0)}-\mathbf{x}^*\Vert^2 + 8\alpha^2L\big(F(\mathbf{x}^{(0)})-F(\mathbf{x}^*)\big) \notag\\
			&\quad + 8\alpha^2LK\big(F(\tilde{\mathbf{x}})-F(\mathbf{x}^*)\big) + 2\alpha^2MC_K + 6\alpha M \sqrt{2\alpha}B_K \notag\\
			&\quad + 2\alpha\sum_{k=1}^K\varepsilon^{(k)} + A_K^2 \Big)^{1/2} \notag\\
			&\le 2A_K + \Vert \mathbf{x}^{(0)}-\mathbf{x}^*\Vert +\sqrt{8\alpha^2L\big(F(\mathbf{x}^{(0)})-F(\mathbf{x}^*)\big)} \notag\\
			&\quad + \sqrt{8\alpha^2LK\big(F(\tilde{\mathbf{x}})-F(\mathbf{x}^*)\big)} + \sqrt{6\alpha M \sqrt{2\alpha} B_K}  \notag\\
			&\quad + \sqrt{2\alpha}B_K + \sqrt{2\alpha^2MC_K} 
		\end{align}
		with $A_K = \sum_{k=1}^K\big(\alpha\mathbb{E}\Vert \mathbf{e}^{(k)}\Vert +\sqrt{2\alpha\varepsilon^{(k)}}\big)$, $B_K = \sum_{k=1}^K\sqrt{\varepsilon^{(k)}}$, $C_K = \sum_{k=1}^K\mathbb{E}\Vert \mathbf{e}^{(k)}\Vert$. Combining \eqref{eq:ExK-x*} and \eqref{eq:Exk-x*} yields
		\begin{align}
			\label{eq:ABCK-final}
			&(2\alpha-8\alpha^2L) \sum_{k=1}^K \mathbb{E}\big(F(\mathbf{x}^{(k)}) - F(\mathbf{x}^*)\big)  + \mathbb{E}\Vert \mathbf{x}^{(K)}-\mathbf{x}^*\Vert^2 \notag \\
			&\quad + 8\alpha^2L\mathbb{E}\big(f(\mathbf{x}^{(K)})-f(\mathbf{x}^*)\big)\notag\\
			&\le \Big(\Vert \mathbf{x}^{(0)} - \mathbf{x}^*\Vert+A_K\Big)^2 + \Big(\sqrt{6\alpha M \sqrt{2\alpha} B_K}+A_K\Big)^2 \notag\\
			&\quad + \Big(\sqrt{2\alpha}B_K+A_K\Big)^2 + \Big(\sqrt{2\alpha^2MC_K}+A_K\Big)^2 \notag\\
			&\quad + 8\alpha^2LK\big(F(\tilde{\mathbf{x}})-F(\mathbf{x}^*)\big) + 8\alpha^2L\big(F(\mathbf{x}^{(0)})-F(\mathbf{x}^*)\big) \notag\\
			&\quad + 2A_K \sqrt{8\alpha^2LK\big(F(\tilde{\mathbf{x}}^{s-1})-F(\mathbf{x}^*)\big)} \notag \\
			&\quad + 2A_K\sqrt{8\alpha^2L\big(F(\mathbf{x}^{(0)})-F(\mathbf{x}^*)\big)} \notag\\
			&\le 2\Vert \mathbf{x}^{(0)} - \mathbf{x}^*\Vert^2 + 16\alpha^2L\big(F(\mathbf{x}^{(0)})-F(\mathbf{x}^*)\big) \notag \notag\\
			&\quad + 16\alpha^2LK\big(F(\tilde{\mathbf{x}})-F(\mathbf{x}^*)\big) + \Big(\sqrt{6\alpha M \sqrt{2\alpha} B_K} +A_K\Big)^2 \notag\\
			&\quad + \Big(\sqrt{2\alpha^2MC_K}+A_K\Big)^2 + \Big(\sqrt{2\alpha}B_K+A_K\Big)^2+ 3A_K^2.
		\end{align}
		\vspace*{-5pt}
		Dividing \eqref{eq:ABCK-final} over $K \theta$ with $\theta = 2\alpha - 8 \alpha^2 L$, we acquire
		\begin{align}
			\label{eq:Fxk-Fx*}
			&\frac{1}{K}\sum_{k=1}^K\mathbb{E}\big(F(\mathbf{x}^{(k)}) - F(\mathbf{x}^*)\big) + \frac{\mathbb{E}\Vert \mathbf{x}^{(K)}-\mathbf{x}^*\Vert^2}{K\theta} \notag\\
			&\quad +\frac{8\alpha^2L}{K\theta}\mathbb{E}\big(F(\mathbf{x}^{(K)})-F(\mathbf{x}^*)\big)\notag\\
			&\le \frac{16\alpha^2L}{\theta} \big(F(\tilde{\mathbf{x}})-F(\mathbf{x}^*)\big) + \frac{2\Vert \mathbf{x}^{(0)} - \mathbf{x}^*\Vert^2}{K\theta} + \frac{3A_K^2}{K\theta}\notag\\
			&\quad +\frac{16\alpha^2L}{K\theta} \big(F(\mathbf{x}^{(0)})-F(\mathbf{x}^*)\big)+ \frac{\left(\sqrt{2\alpha}B_K+A_K\right)^2}{K\theta} \notag\\
			&\quad + \frac{\left(\sqrt{6\alpha M \sqrt{2\alpha} B_K}+A_K\right)^2}{K\theta}
			+ \frac{\left(\sqrt{2\alpha^2MC_K}+A_K\right)^2}{K\theta}.
		\end{align}
		
		According to Assumption \ref{ass:summable} which guarantees the summability of $\mathbb{E}{\mathbf{e}^{(k)}}$ and $\sqrt{\epsilon^{(k)}}$, $A_K$, $B_K$ and $C_K$ will be finite. Therefore, the sum $\Big(3A_K^2
		+ \big(\sqrt{2\alpha}B_K+A_K\big)^2 + \big(\sqrt{6\alpha M \sqrt{2\alpha} B_K}+A_K\big)^2
		+ \big(\sqrt{2\alpha^2MC_K}+A_K\big)^2\Big)/\theta$ can be bounded by the constant $C$. Then we utilize notation $C$ to simplify \eqref{eq:Fxk-Fx*} and obtain Proposition 3.
		
		\bibliographystyle{IEEEtran}
		\bibliography{ref_arxiv}

	\end{document}